\documentclass{article}

\usepackage{ICLR2026/iclr2026_conference,times}
\usepackage{hyperref}       
\usepackage{url}            
\usepackage{mathtools}
\usepackage{booktabs}       
\usepackage{amsfonts}       
\usepackage{nicefrac}       
\usepackage{microtype}      
\usepackage[table]{xcolor}  
\usepackage[acronym]{glossaries}
\usepackage{wrapfig}
\usepackage{amsthm}
\usepackage{subcaption}
\usepackage{lipsum}
\usepackage{amsmath}
\usepackage{graphicx}       
\usepackage{cleveref}

\crefname{appendix}{Appendix}{Appendices}
\Crefname{appendix}{Appendix}{Appendices}

\makeatletter
\AddToHook{cmd/appendix/before}{%
  \def\cref@section@alias{appendix}%
  \def\cref@subsection@alias{appendix}%
}
\makeatother

\usepackage{dirtytalk}
\usepackage{fp}
\usepackage{float}
\usepackage{placeins}
\usepackage{enumitem}

\usepackage{multirow}
\usepackage{etoc} 

\makeatletter
\@ifpackageloaded{glossaries}{\glsdisablehyper}{}%
\makeatother

\makeglossaries
\glsdisablehyper

\newacronym{sota}{SOTA}{state-of-the-art}
\newacronym{ad}{AD}{auto-differentiation}

\newacronym{lj}{LJ}{Lennard-Jones}

\newacronym{fbm}{fBM}{fractional brownian motion}
\newacronym{rl}{RL}{Riemann-Liouville}
\newacronym{ou}{OU}{Ornstein-Uhlenbeck}

\newacronym{rbergomi}{rBergomi}{rough Bergomi}

\newacronym{rge}{RGE}{rotational geodesic error}
\newacronym{mse}{MSE}{mean squared error}
\newacronym{mmd}{MMD}{maximum mean discrepancy}

\newacronym{ode}{ODE}{ordinary differential equation}
\newacronym{pde}{PDE}{partial differential equation}
\newacronym{sde}{SDE}{stochastic differential equation}
\newacronym{cde}{CDE}{controlled differential equation}
\newacronym{yde}{YDE}{Young differential equation}
\newacronym{rde}{RDE}{rough differential equation}
\newacronym{rsde}{RSDE}{rough stochastic differential equation}
\newacronym{spde}{SPDE}{stochastic partial differential equation}

\newacronym{rk}{RK}{Runge--Kutta}
\newacronym{rkmk}{RKMK}{Runge--Kutta--Munthe--Kaas}

\newacronym{sg}{SG}{Savitzky--Golay}

\newacronym{node}{NODE}{neural ordinary differential equation}
\newacronym{ncde}{NCDE}{neural controlled differential equation}
\newacronym{nrde}{NRDE}{neural rough differential equation}
\newacronym{bnrde}{B-NRDE}{branched neural rough differential equation}

\newacronym{nspde}{NSPDE}{neural stochastic partial differential equation}
\newacronym{dlrnet}{DLR-Net}{Deep Latent Regularity Net}
\newacronym{nors}{NORS}{Neural Operator with Regularity Structure}

\newacronym{gl}{GL}{Grossman--Larson}
\newacronym{bck}{BCK}{Butcher--Connes--Kreimer}
\newacronym{mkw}{MKW}{Munthe--Kaas--Wright}

\newacronym{smiles}{SMILES}{Simplified Molecular-input Line-entry System}
\newacronym{dft}{DFT}{density functional theory}
\newacronym{md}{MD}{molecular dynamics}

\newacronym{spdgraph}{SPD}{shorest path distance}

\newacronym{spdman}{SPD}{symmetric positive definite}
\newacronym{aiman}{AI}{affine invariant}
\newacronym{bw}{BW}{Bures-Wasserstein}

\newacronym{mlp}{MLP}{multilayer perceptron}
\newacronym{nlp}{NLP}{natural language processing}
\newacronym{llm}{LLM}{large language models}
\newacronym{ffn}{FFN}{feedforward neural network}
\newacronym{cnn}{CNN}{convolutional neural network}
\newacronym{mpnn}{MPNN}{message passing neural network}

\newacronym{mt}{MT}{multitask}
\newacronym{zs}{ZS}{zero-shot}

\newacronym{relu}{ReLU}{rectified linear unit}
\newacronym{gelu}{GELU}{Gaussian error linear unit}
\newacronym{silu}{SiLU}{sigmoid linear unit}

\newacronym{gnn}{GNN}{graph neural network}
\newacronym{gtn}{GTN}{graph transformer neural network}
\newacronym{egnn}{EGNN}{Equivariant Graph Neural Network}
\newacronym{gat}{GAT}{graph attention network}
\newacronym{gatv2}{GATv2}{graph attention network v2}

\newacronym{mha}{MHA}{multi-headed scaled dot-product attention}
\newacronym{rope}{RoPE}{Rotary Position Embedding}
\newacronym{trope}{T-RoPE}{Temporal Rotary Position Embedding}
\newacronym{nope}{NoPE}{no positional encoding}

\newacronym{rwpe}{RWPE}{random walk positional encoding}
\newacronym{rrwp}{RRWP}{relative random walk probabilities}

\newacronym{no}{NO}{neural operator}
\newacronym{atoms}{ATOM}{Atomistic Transformer Operator for Molecules}
\newacronym{fno}{FNO}{Fourier Neural Operator}
\newacronym{egno}{EGNO}{Equivariant Graph Neural Operator}
\newacronym{gnot}{GNOT}{General Neural Operator Transformer}
\newacronym{hnca}{HNCA}{Heterogenous Normalised Cross-Attention}

\newacronym{ks}{KS}{Kolmogorov-Smirnov}

\newacronym{omd}{OMD}{Oxford Multi-motion Dataset}

\let\cite\citep
\let\textcite\citet

\newtheorem{theorem}{Theorem}[section]

\newtheorem{proposition}[theorem]{Proposition}
\newtheorem{definition}[theorem]{Definition}

\hypersetup{colorlinks=true,linkcolor=red!70!black,linktocpage=false,citebordercolor=blue!70!black,citecolor=blue!70!black,anchorcolor=blue!70!black}

\title{ATOM: A Pretrained Neural Operator for\\ Multitask Molecular Dynamics}
\author{Luke Thompson \\
  The University of Sydney
  \And
  Davy Guan \\
  Data61 CSIRO 
  \And
  Dai Shi \\
  University of Cambridge
  \And
  Slade Matthews\thanks{These authors contributed equally as senior authors.} \\
  The University of Sydney 
  \And
  Junbin Gao\footnotemark[1] \\
  The University of Sydney 
  \And
  Andi Han\footnotemark[1] \ \thanks{Correspondence to: \texttt{andi.han@sydney.edu.au}.} \\
  The University of Sydney 
}

\iclrfinalcopy 

\begin{document}

\maketitle
\thispagestyle{fancy}
\pagestyle{fancy}
\lhead{Published as a conference paper at ICLR 2026}
\begin{abstract}

\Gls{md} simulations underpin modern computational drug discovery, materials science, and biochemistry. Recent machine learning models provide high-fidelity \gls{md} predictions without the need to repeatedly solve quantum mechanical forces, enabling significant speedups over conventional pipelines. Yet many such methods typically enforce strict equivariance and rely on sequential rollouts, thus limiting their flexibility and simulation efficiency. They are also commonly single-task, trained on individual molecules and fixed timeframes, which restricts generalization to unseen compounds and extended timesteps. To address these issues, we propose \gls{atoms}, a pretrained transformer neural operator for multitask molecular dynamics. \gls{atoms} adopts a quasi-equivariant design that requires no explicit molecular graph and employs a temporal attention mechanism, enabling accurate parallel decoding of multiple future states. To support operator pretraining across chemicals and timescales, we curate TG80, a large, diverse, and numerically stable \gls{md} dataset with over 2.5 million femtoseconds of trajectories across 80 compounds. \gls{atoms} achieves state-of-the-art performance on established single-task benchmarks, such as MD17, RMD17 and MD22. After multitask pretraining on TG80, \gls{atoms} shows exceptional zero-shot generalization to unseen molecules across varying time horizons. We believe \gls{atoms} represents a significant step toward accurate, efficient, and transferable molecular dynamics models.

\end{abstract}

 \addtocontents{toc}{\protect\setcounter{tocdepth}{-1}}
\section{Introduction}   
\glsresetall

        
        
    

\Gls{md} serves as a computational microscope of atomic motion and is now integral to drug discovery and materials science pipelines \cite{dror_biomolecular_2012, de_vivo_role_2016}. In ab initio molecular dynamics, quantum-mechanical \gls{dft} is used to compute atomic forces, and the resulting equations of motion are integrated to generate high-fidelity trajectories. However, \gls{dft}'s computational complexity scales at least cubically with the number of atoms, and relies on double-precision arithmetic that limits GPU acceleration \cite{kresse_efficiency_1996, stein_double-hybrid_2020, li_scaling_2024}. 

Neural approaches have recently emerged as a promising solution to the scalability bottleneck. \textit{Equivariant} architectures, in particular, encode physical symmetries to model interatomic dynamics, achieving ab initio-level accuracy at significantly reduced computational cost  \cite{batzner_e3-equivariant_2022, musaelian_learning_2022, batatia_design_2022, batatia_mace_2023, xu_equivariant_2024}. While equivariance is often deemed essential for improving generalization, strict symmetry preservation involves substantial tradeoffs \cite{xu_equivariant_2024, schreiner_implicit_2023}. Architectures that enforce exact equivariance at every layer often increase computational overhead, restrict model expressivity, and complicate optimization \cite{fuchs_se3-transformers_2020, brehmer_geometric_2023, elhag_relaxed_2025}. It is unclear whether symmetry constraints can be relaxed without sacrificing accuracy for molecular dynamics.

Furthermore, most existing methods for molecular dynamics are \textit{autoregressive}, predicting the next state based on the current one \citep{kohler_equivariant_2019,fuchs_se3-transformers_2020,thiemann2025force}.  Autoregressive approaches often struggle to capture long-horizon temporal dependencies and accumulate error as the prediction horizon grows \cite{bengio_scheduled_2015, bergsma_sutranets_2023, taieb_bias_2016}. Inference speeds are also constrained by the need for sequential integration, failing to exploit modern, highly parallel compute architectures. One exception is \gls{egno} \cite{xu_equivariant_2024}, which models the entire trajectory with neural operator learning. Nevertheless, \gls{egno} enforces strict equivariance and is single-task in nature, i.e., it is trained and evaluated on trajectories of each molecule separately with a fixed time horizon, which limits \textit{zero-shot generalization} to unseen molecules or timeframes. 

\textbf{Our Main Contributions.} In this work, we address the above issues regarding equivariance, autoregression, and zero-shot generalization within a unified framework, which we call \textit{\gls{atoms}}. To this end, we propose a pre-trained neural operator with a transformer backbone for molecular dynamics and introduce a new \gls{md} dataset, TG80, which is both chemically diverse and numerically stable for multitask pretraining and benchmarking. 
\begin{itemize}[leftmargin=0.15in]
    \item \textit{Design innovations}. \gls{atoms} is \textit{quasi-equivariant} by employing an equivariant lifting layer that produces symmetry-aware features, while allowing subsequent transformer blocks to be unconstrained for flexibility and expressiveness. 
    Unlike autoregressive models, \gls{atoms} allows \textit{parallel decoding} of molecule states across multiple timesteps, directly learning the trajectory operator. By encoding time lags via a novel temporal rotary position embedding, \gls{atoms} enhances temporal interpolation and extrapolation, enabling robust predictions across multiple time horizons. Finally, \gls{atoms} requires no predefined molecular graph and operates directly on \textit{point clouds}, naturally accommodating long-range spatial interactions without the need for hand-crafted connectivity. 

    \item \textit{Performance highlights}. \gls{atoms} sets new state-of-the-art on single-task \gls{md} benchmarks. For larger, sparsely connected molecules in MD22, \gls{atoms} significantly outperforms existing graph-based baselines by capturing the long-range atomic interactions. In the multitask regime, we pretrain \gls{atoms} on TG80 trajectories from multiple molecules and varying timeframes, demonstrating significant zero-shot transfer to both unseen molecules and timesteps, improving existing baselines by 39.75\% on average. This achieves performance on par with existing specialized baselines tailored for such molecules and timeframes.  To the best of our knowledge, this is the first method that demonstrates such generalization capability in molecular dynamics.
\end{itemize}

We believe our work represents a shift in molecular dynamics modeling, where we demonstrate the potential of quasi-equivariance designs and zero-shot generalization to out-of-domain systems, which is enabled by the comprehensive TG80 \gls{md} dataset. \gls{atoms} and TG80 are available at \href{https://github.com/luke-a-thompson/ATOM}{this repository}.

\section{Related work}
\textbf{Equivariant Neural Networks.}
Equivariance (to transformations such as rotation, reflection, and translation) has emerged as an essential physics-informed prior for deep learning models on molecular data \cite{bronstein_geometric_2021, duval_hitchhikers_2023}. 
Early works employed convolutional approaches to achieve translation equivariance in E(3) \cite{weiler_3d_2018, wu_pointconv_2020} or tensor product attention and spherical harmonics to enforce roto-translational equivariance in SE(3) \cite{fuchs_se3-transformers_2020,thomas_tensor_2018}. In contrast, \gls{mpnn} frameworks, such as \gls{egnn} and others \cite{garcia_satorras_en_2021, gasteiger_gemnet_2021, huang_equivariant_2022}, achieve equivariance by operating on strictly equivariant features, such as inter-node distances and directions. While effective, \glspl{mpnn} typically assume a fixed molecular graph. This is problematic when the underlying structure contains non-local interactions and dynamic bonding effects (e.g., resonances, transient interactions), which render predefined graphs inaccurate over time \cite{knutson_dynamic_2022,luo_predicting_2021}. To address this issue, we model molecules as point clouds, with our attention represented as a fully connected graph that allows unrestricted information propagation across the molecule.


{\textbf{Time-coarsened Molecular Dynamics}
Time coarsening is a coarse-graining method which preserves molecular structure, but compresses many short integration steps into a few large-stride updates to reduce the cost of long-time simulation \cite{kmiecik_coarse-grained_2016}.
Stochastic coarse-graining approaches often learn transition kernels on configuration space, bypassing explicit integration of the equations of motion. \textcite{klein_timewarp_2023} learns such a kernel with a normalizing flow and uses it as an MCMC proposal targeting the Boltzmann distribution, \textcite{hsu_score_2024} uses a conditional diffusion model to learn a transition probability matrix, and \textcite{yu_unisim_2025} uses flow-matching to learn a vector field transporting current states to future states.
Closer to our framework, deterministic methods such as MDNet \cite{zheng_learning_2021} and TrajCast \cite{thiemann2025force} learn a GNN and \gls{egnn}, respectively, which autoregressively predict fixed strides 10-100 times larger than those of \gls{md} integrators. \textcite{bigi_flashmd_2025} incorporates Hamiltonian structure and explicit energy-conservation. Most of the methods require direct force learning and are sequential in nature, while \Gls{atoms} may be interpreted as a \textit{force-free} deterministic coarse-graining approach, wherein temporal pushforward is approximated by a learned propagation operator which is decoded \textit{in parallel}.}

\textbf{Neural Operators.}
Neural operators are deep learning methods for learning operators between function spaces \cite{kovachki_neural_2021}. 
A wide variety of architectures have been proposed for such operator learning. Notably, \gls{fno} \cite{li_fourier_2021} learns an operator in the Fourier domain, while its derivatives G-FNO \cite{helwig_group_2023} and PINO \cite{li_physics-informed_2023}, respectively, add group equivariance and physics-informed properties. \citet{xu_equivariant_2024} bridges this framework with molecular dynamics by recasting the task as learning a propagation operator that evolves historical atomic positions into their future configurations. Specifically, \gls{egno} \cite{xu_equivariant_2024} integrates \gls{egnn} and \gls{fno} layers to learn dynamic trajectories, capturing spatial and temporal correlations.
Recently, transformer neural operators \cite{bryutkin_hamlet_2024,hao_gnot_2023,li_transformer_2023} have surpassed the performance of \gls{fno} in most \gls{pde} tasks. Notably, OFormer \cite{li_transformer_2023} uses a linear Galerkin-type attention mechanism, which omits the softmax and instead interprets the latent column vectors as basis functions. \gls{gnot} \cite{hao_gnot_2023} employs a novel subquadratic cross-attention methodology to integrate multiple feature types (e.g., shape and point relationships) into their transformer blocks. With \gls{atoms}, we unify the \gls{md} problem formulation and temporal discretization approach introduced by \gls{egno} with the increased representational power of transformers in operator settings.


\textbf{MD Benchmarks.}
Research on graph machine learning for molecular dynamics suffers from poor benchmarking \cite{bechler-speicher_position_2025}. 
For example, despite the fact that MD17 Benzene exhibits non-physical noise approximately 1000 times higher compared to other compounds \cite{christensen_role_2020}, it is still regularly employed to benchmark new models \cite{bihani_egraffbench_2023,huang_equivariant_2022,liao_equiformer_2023, xu_equivariant_2024}. The practical relevance of single-task learning on these datasets is also dubious, as predicting trajectories for molecules with existing numerical solutions offers minimal benefit. We believe the strengths of neural approaches emerge in transfer learning, where models generalize to unseen compounds, thereby circumventing the computational costs associated with explicit numerical simulations. This motivates our development of TG80 to facilitate multitask dynamics learning across molecular systems.

\section{Atomistic Transformer Operator for Molecules (ATOM)}
\label{sect:atom}

In this section, we first introduce the problem formulation (\Cref{sect:problem}) and then propose the framework of \gls{atoms} by introducing the key model and training designs (\Cref{sect:main_model}). We then discuss the multitask pretraining for \gls{atoms} and introduce TG80 \gls{md} dataset (\Cref{sect:multi-pretrain}).

\subsection{Problem Formulation}
\label{sect:problem}
We follow \cite{xu_equivariant_2024} to cast molecular dynamics prediction as operator learning. We model a molecule of $N$ atoms as a point cloud in $\mathbb R^3$, which we denote as $\mathcal{G}^{(t)}$ for a given system state time $t$. In particular, we write $\mathcal{G}^{(t)} = (\mathbf x_i^{(t)}, \mathbf v_i^{(t)})_{i= 1}^N$ that represent molecules in terms of the atom positions $\mathbf{x}$ and velocities $\mathbf{v}$.
Our objective is to predict a future trajectory $\mathcal{G}^{(t + \Delta t)}$, where $\Delta t \in [0, \Delta T]$. 

Similar to \cite{xu_equivariant_2024}, we focus on predicting the position states only. Let $\mathcal{U} \colon [0, \Delta T] \rightarrow \mathbb R^{N \times  3}$ be the trajectory function mapping $\Delta t$ to $U(\Delta t) \in \mathbb R^{N \times  3}$ representing molecule positions $\Delta t$ in the future. We assume a solution operator $F^\dagger \colon \mathcal{G}^{(t)} \rightarrow \mathcal{U}$ exists which provides the underlying future trajectory given system states at $t$. Thus, the goal of molecular dynamics prediction becomes training a neural operator $F_\theta(\mathcal{G}^{(t)})$ to approximate the target trajectory function $F^\dagger(\mathcal{G}^{(t)})$: $\min_\theta \mathbb E_{\mathcal{G}^{(t)}} \mathcal{L} \big(F_\theta(\mathcal{G}^{(t)}) (t) , F^\dagger(\mathcal{G}^{(t)})(t) \big)$,
for some loss function $\mathcal{L} \colon \mathcal{U} \times \mathcal{U}\rightarrow \mathbb R$. Here, expectation is with respect to the different initial states.
By discretizing over the temporal domain and considering $L_2$ loss, we optimize the neural operator with a discretized temporal sampling of the states:
\begin{align}
    &\min_{\theta} \frac{1}{P} \sum_{p=1}^P \mathbb{E}_{\substack{
            \mathcal{G}^{(t)}  \\
            }
        } 
         \left\|
            F_\theta\left( \mathcal{G}^{(t)}  \right)(\Delta t_p) - F^\dagger \left (\mathcal{G}^{(t)} \right) (\Delta t_p)
        \right\|_{2}^{2}.
        \label{eq:objective_function}
\end{align}
where $\{ \Delta t_1, ..., \Delta t_p \}$ are discrete timesteps. We replace the true future state $F^\dagger(\mathcal{G}^{(t)})(\Delta t_p)$ with the known future ground truth node positions $\mathbf{x}^{(t+\Delta t_p)}$ for $\Delta t_p \in [0, \Delta T]$.

{
\textbf{Quasi-equivariance.}
\label{def:quasi_equivariance}
We formally define quasi-equivariance, motivated by \cite{elhag_relaxed_2025}.
\begin{definition}[$\varepsilon$-quasi-equivariance]
We call a function $f: \mathcal{X} \rightarrow \mathcal{Y}$, $\varepsilon$-quasi-equivariant with respect to group $G$ if it satisfies $\mathbb E_{x \in \mathcal{X}} \|  \int_G f( \phi(g)(x) ) d \mu (g)  - \int_G \rho(g) (f(x)) d\mu(g)
\| \leq \varepsilon$, 
where $\mu$ denotes the normalized Haar measure. 
\end{definition}
In practice, we approximate the group integration with Monte Carlo samples from $G$. 
}

\textbf{Single- and multitask.}
\label{sec:st_and_mt}
Unlike prior works \cite{schreiner_implicit_2023, xu_equivariant_2024}, we consider both single-task and multitask settings. \textit{Single-task} refers to the case where a separate model is independently trained and evaluated on each molecule and fixed timeframes. This corresponds to the conventional practice in molecular dynamics benchmarks. \textit{Multitask} instead pretrains one unified model on several molecules across varying time lags and evaluates out-of-domain trajectories on unseen molecules, thereby directly testing zero-shot cross-molecule generalization. Under a multitask setting, the objective \eqref{eq:objective_function} computes the expectation over trajectories of multiple molecules.



\subsection{ATOM Model and Training Design}
\label{sect:main_model}


Here we outline the pipeline of \gls{atoms}. At its core is an \textit{equivariant lifting} layer (\Cref{sec:equivariant_lifting}), which maps atomic positions, velocities and their phase features into a richer embedding space while preserving symmetry under the Euclidean group \(E(3)\). The lifted embeddings are then processed by the \gls{atoms} attention block, which applies \textit{heterogeneous attention} over positions, velocities, and phase features with chemical augmentation (\Cref{sect:atom_attn}).  To capture temporal dynamics, we incorporate a \textit{temporal rotary position embedding} (T-RoPE) (\Cref{sect:atom_attn}) that depends only on time lags and is shared across atoms, ensuring translation invariance in time and permutation invariance within each molecule. {The parameterized \gls{atoms} can be written by 
\begin{equation*}
    F_\theta \coloneqq \mathcal{P} \circ \sigma(\mathcal{K}_L) \circ \cdots \circ \sigma(\mathcal{K}_1) \circ \mathcal{Q}
\end{equation*}
where $\mathcal{Q}$, $\mathcal{P}$ denotes the equivariant lifting and projection operators respectively. $\mathcal{K}_l, l = 1,..., L$ are the data-dependent kernels induced by cross attention (See Appendix \ref{app:kernel}), and $\sigma$ denotes some nonlinear activation function.
}

Finally, to counter numerical noise in training trajectories, we inject randomly sampled position and velocity perturbations during training (\Cref{sec:label_noise}), which improves robustness and acts as a regularizer against overfitting. 
The overall pipeline of \gls{atoms} is in \Cref{fig:atom-arch}. 


\begin{figure}[t]
    \centering
    \includegraphics[width=1\linewidth]{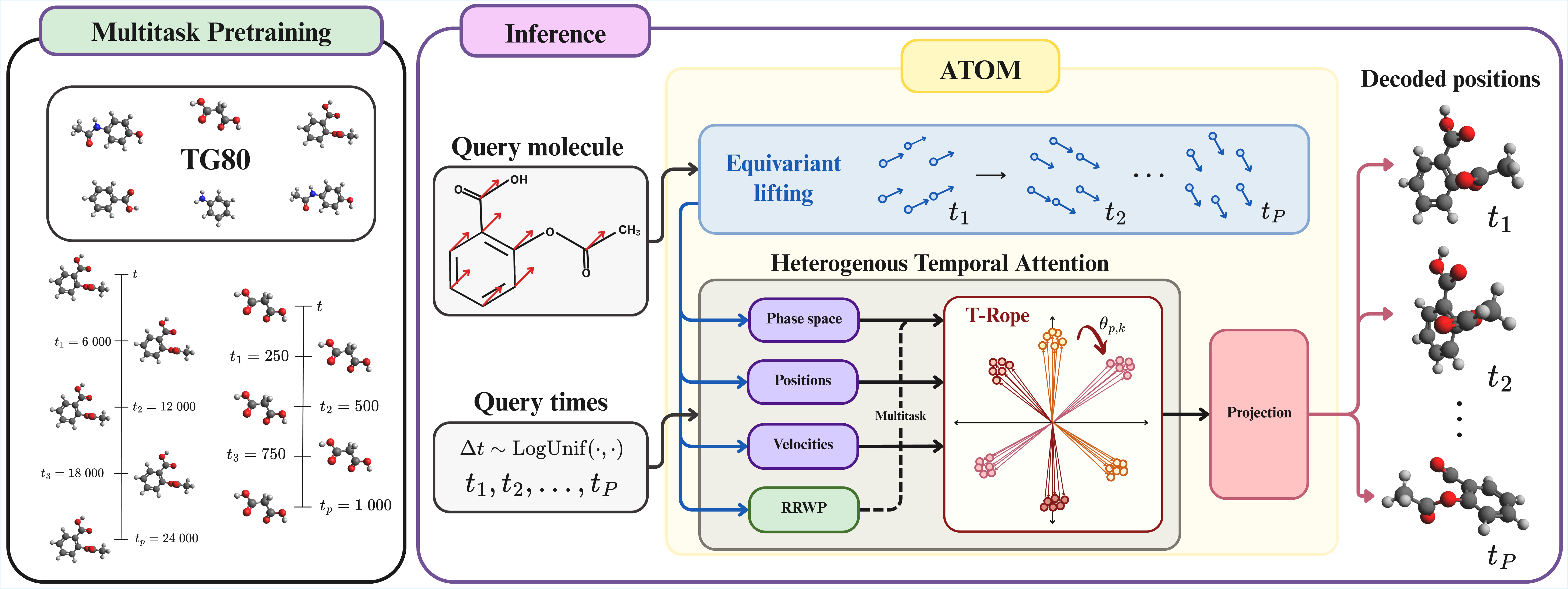}
    \caption{\textbf{\gls{atoms} Pipeline}. We pretrain \gls{atoms} on the TG80 dataset across multiple molecules with stochastic time lags. At inference, \gls{atoms} takes a query molecule and timestamps and directly outputs corresponding molecular states. 
    }
    \label{fig:atom-arch}
\end{figure}

\subsubsection{\textit{E}(3) Equivariant Lifting}
\label{sec:equivariant_lifting}
To model atomic states in a symmetry-respecting way, each atom is encoded with its 3D position and velocity, augmented with their norms: $\mathbf{x} = (x,y,z,\sqrt{x^2+y^2+z^2}), \, \mathbf{v} = (v_x, v_y, v_z, \sqrt{v_x^2+v_y^2+v_z^2})$.
To construct higher-dimensional features that remain consistent with $E(3)$ symmetry, we apply \textit{equivariant lifting} that maps the inputs through learnable functions that preserve group actions.
Specifically, we use \(E(3)\)-equivariant linear layers \cite{geiger_e3nn_2022} that lift the position and velocity vectors to a feature space
wherein they satisfy equivariance constraints by construction. We further construct phase space featrues for each atom by augmenting the position and velocity vectors with atomic number, which is subsequently processed by a learnable equivariant layer to obtain a lifted representation. The final lifted embedding for a molecule is given by $({\mathbf X}, {\mathbf{V}}, {\mathbf{Z}}) \in \mathbb R^{3 \times N P \times d_v}$ corresponding to position, velocity and phase features. The second dimension aggregates nodes and time for attention and $d_v$ is the embedding space dimension. 

We highlight that after the equivariant lifting layer, we do not enforce equivariance in the subsequent transformer blocks. This relaxation improves performance over fully equivariant designs while still showing robustness to random trajectory rotations versus non-equivariant baselines (see \Cref{sec:ablations}).

\subsubsection{ATOM Heterogeneous Temporal Attention}
\label{sect:atom_attn}

We employ a heterogeneous temporal attention mechanism to enable mixing between multiple features $({\mathbf{X}}, {\mathbf{V}}, {\mathbf{Z}}) \in \mathbb R^{3 \times N P \times d_v}$ across spatial and temporal dimensions.
\FPeval{\PctImproveCrossAttvsStdAtt}{round((7.39-6.92)/7.39*100,2)} 
We use the phase space embedding ${\mathbf Z}$ as the query and attend to the key-value pairs formed from all features ${\mathbf{X}}, {\mathbf{V}}, {\mathbf{Z}} \in \mathbb R^{NP \times d_v}$. In \Cref{fig:MD17_ablation}, we show that this improves performance by \PctImproveCrossAttvsStdAtt\% over standard self-attention for single-task prediction. To encode temporal information, we introduce Temporal RoPE (T-RoPE), adapting \gls{rope} \citep{su_roformer_2023} to irregular time lags by scaling the phase rotations according to cumulative timestamps constructed from per-step increments ${ \Delta t}$.

Let the hidden dimension per head be $d_h$ (even). We define frequencies $\omega_k=b^{-2k/d_h}$ for $k=0,\dots,d_h/2-1$.
Given per-step time increments $\{\Delta t_p\}_{p=1}^{P}$, we build timestamps
$t_p=t+\sum_{r=1}^{p}\Delta t_r$, and assign a \emph{single} rotation to all $N$ atoms at timestep $p$: $\mathbf R_{p}=\operatorname{diag}\left(\mathbf R(\theta_{p,0}),\dots,\mathbf R(\theta_{p,d_h/2-1})\right) \in \mathbb R^{d_h \times d_h}$, where $\theta_{p,k}=\frac{\omega_k}{\tau}\,\bigl(t_p-t_0\bigr)$ and
$\mathbf R(\theta) \in \mathbb R^{2 \times 2}$ is the rotation matrix with angle $\theta$ and $\tau > 0$ is a timescale hyperparameter. Suppose the query molecule state at time $p$ is given as $\mathbf Q_p \in \mathbb R^{N \times d_h}$ and key molecule state at time $p'$ is $\mathbf K_{p'} \in \mathbb R^{N \times d_h}$. We apply $\mathbf R_p, \mathbf R_{p'}$ to $\mathbf Q_p, \mathbf K_{p'}$ respectively so that the rotary dot product $\mathbf Q_p \mathbf R_p (\mathbf K_{p'} \mathbf R_{p'})^\top$  
depends only on the time interval $t_{p'}-t_p$. This makes attention \textit{translation invariant} in time, which allows for interpolation and extrapolation across irregular increments $\{\Delta t_p\}$. In addition, sharing the same $\mathbf R_p$ across all $N$ atoms in a molecule ensures \textit{permutation-invariance} within a timestep.
For aggregated query and key matrices $\mathbf Q, \mathbf K \in \mathbb R^{NP \times d_h}$, we denote the application of temporal \gls{rope} across $P$ timesteps and $N$ atoms as $\operatorname{T-RoPE}(\mathbf Q), \operatorname{T-RoPE}(\mathbf K) \in \mathbb R^{NP \times d_h}$.

Specifically, a single-head attention layer of \gls{atoms} computes 
\begin{equation*}
    \sum_{\mathbf F   \in \{{\mathbf{X}}, {\mathbf{V}}, {\mathbf{Z}} \}} \gamma_{\mathbf F}  \, \operatorname{softmax} \Big( \frac{ \operatorname{T-RoPE}(Q({\mathbf{Z}})) \, \operatorname{T-RoPE}( K(\mathbf F) )^\top}{\sqrt{d_h}} \Big) 
     V(\mathbf F),
\end{equation*}
where $Q(\cdot), K(\cdot), V(\cdot)$ represent the query, key and value projections. 
We introduce learnable weights $\gamma_{\mathbf F}$ to modulate the relative importance of each feature. 
In \Cref{appx:kernel_integral_cross_att}, we show that heterogeneous attention is equivalent to a kernel integral operator and discuss its properties.

\subsubsection{Training with Label Noise Regularization}
\label{sec:label_noise}
Many \gls{dft} datasets are inherently noisy \cite{christensen_role_2020}, and \gls{md} models can overfit to this noise. Motivated by the regularization effect of label noise \cite{damian_label_2021,haochen_shape_2020}, we augment the observed node  positions \(\mathbf{x}\) and velocities \(\mathbf{v}\) by random Gaussian noise \(\boldsymbol{\xi}_\mathbf{x}, \boldsymbol{\xi}_\mathbf{v} \sim \mathcal{N}(\mathbf{0}, \sigma^2 \mathbf{I})\) during training. 
Let ${\mathcal{G}}^{(t)}_{\boldsymbol{\xi}} = ( \mathbf x_i^{(t)} + \boldsymbol{\xi}_{x, i}, \mathbf v^{(t)} + \boldsymbol{\xi}_{v,i})$ be the noised initial state at time $t$. We minimize the following regularized loss
\begin{align*}
    &\min_{\theta} \frac{1}{P} \sum_{p=1}^P \mathbb{E}_{ 
            \mathcal{G}^{(t)}, \boldsymbol{\xi}, \boldsymbol{\xi}_x^p} 
         \left\|
            F_\theta\left( \mathcal{G}^{(t)}_{\boldsymbol{\xi}}   \right)(\Delta t_p) - (\mathbf{x}^{(t+ \Delta t_p) } + \boldsymbol{\xi}_x^p)
        \right\|_{2}^{2}.
\end{align*}
A similar strategy has also appeared in \gls{gnn}-based  \gls{md} models and neural operator pretraining \cite{dauparas_robust_2022, zhou_strategies_2024, hao_dpot_2024}. 
We only apply noise augmentation during training and evaluate on the unperturbed ground-truth trajectories.


{\textbf{Comparison to \gls{egno}}. We highlight that \gls{atoms} adopts different design choices compared to \gls{egno}. First, EGNO is an EGNN operating on fixed bond connectivity, whereas \gls{atoms} uses an E(3)-equivariant lifting layer followed by globally connected point-cloud attention, which better handles long-range and sparsely bonded interactions. Second, EGNO is strictly equivariant end-to-end, while ATOM is quasi-equivariant, enforcing equivariance only in the lifting stage and relaxing it in deeper transformer layers, which our ablations show improves accuracy. Third, EGNO models time via Fourier temporal convolution, whereas ATOM uses Temporal RoPE, allowing translation-invariant handling of irregular time gaps and stronger temporal extrapolation. T-RoPE also uniquely allows modifying the time-horizon \(\Delta T\) \textit{at inference} by modulating the rotary phases (\Cref{sect:atom_attn}). Consequently, a pretrained \gls{atoms} can be evaluated at arbitrary \(\Delta T\) values without retraining.}

\subsection{Multitask \gls{atoms} Pretraining and TG80 Dataset}
\label{sect:multi-pretrain}

\vspace{-0.1em}

This section adapts \gls{atoms} for the multitask setting, where the aim is to predict future trajectories for unseen molecules. In order to more effectively distinguish molecules, we construct a radius graph of 1.6 \AA{} based on atomic positions, and apply \textit{random walk positional encoding} \cite{ma_graph_2023, lobato_highs_2021} to augment the phase vector $\mathbf{z}$. We describe the process in detail in \Cref{sec:positional_encoding} and highlight that such a graph depends only on atomic positions, not chemical bonds. 

During multitask training, each mini-batch contains trajectories from multiple molecules. 
In addition, we perform random sampling for the time lags $\Delta t$ from a log-uniform distribution between $\Delta t_{\rm min}$ and $\Delta T$, namely $\Delta t \sim {\rm LogUnif}(\Delta t_{\rm min}, \Delta T)$. This aims to enhance the robustness of interpolation and extrapolation in the temporal domain, a consideration that has been similarly explored in \citep{schreiner_implicit_2023}. Let $\mathcal{M}$ denote the set of training molecules and let $\mathcal{G}_m^{(t)}$ represent the state of molecule $m \in \mathcal{M}$ at timestamp $t$. We can write the pretraining multitask objective as 
\begin{equation*}
\min_{\theta}\;
\frac{1}{|\mathcal{M}|}\sum_{m\in\mathcal{M}}
\mathbb{E}_{\mathcal{G}^{(t)}_{m},\, \Delta t \sim \mathrm{LogUnif}(\Delta t_{\min},\, \Delta T)}
\left\|
F_\theta\!\left(\mathcal{G}^{(t)}_{m}, \Delta t\right)(\Delta t) - \mathbf{x}^{(t+\Delta t)}_{m}
\right\|_2^2,
\end{equation*}
where we take expectation with respect to initial states of multiple molecules in the training set, as well as the time lags. Here, \gls{atoms} also takes a time lag, $\Delta t$, as input, to modulate T-RoPE phase.





\textbf{TG80 Dataset.}
To facilitate pretraining of our neural operator, we introduce \textbf{TG80}, a superset of the MD17 dataset. The initial seed set comprises 40 molecules: 8 MD17 compounds and 32 additional drug-like molecules selected through expert review. We then augment the seed molecules with structurally similar molecules from the PubChem dataset of 173 million compounds \cite{bolton_pubchem3d_2011}. 
%
Accepted candidates had an ECFP-4 Tanimoto similarity between 0.875 and 0.925 to at least one seed molecule, and no more than 0.80 similarity to previously accepted molecules, alongside other criteria detailed in \Cref{sec:tg80_algorithm} \cite{greg_landrum_rdkitrdkit_2025, rogers_extended-connectivity_2010, rogers_computer_1960}. These thresholds follow common practice in the literature, balancing diversity while avoiding collapse into overly narrow chemical subspaces \cite{matter_selecting_1997, menke_natural_2021, eastman_spice_2023, harper_research_2004, zhang_activity_2023}.

We generate all trajectories using ORCA V6.01 \cite{neese_software_2022} with the PBE functional \cite{perdew_generalized_1996}, def2-SVP basis set \cite{weigend_balanced_2005}, \(\Delta\)4 dispersion corrections \cite{caldeweyher_generally_2019, caldeweyher_extension_2020, wittmann_extension_2024} at one femtosecond resolution, 300K temperature, in vacuum. This resembles an enhanced RMD17, with more modern dispersion corrections to improve stability and allow for a larger step size \cite{christensen_role_2020}. As a result, TG80 exhibits \textit{more diverse dynamics} and \textit{improved numerical stability}, with no compound exceeding 50~\AA{} center-of-mass drift in \Cref{fig:numerical_stability}\footnote{Simulations ran on 32 AMD EPYC 7543 cores with 256\,GB RAM per molecule, totalling 806{,}400 CPU-hours (quoted market cost USD~150\;000).}.
\vspace{-0.3em}
\section{Experiment Results}
\vspace{-0.5em}
\textbf{Metrics.}
We use \textit{State-to-trajectory}  ($\operatorname{S2T}$) and \textit{state-to-state} ($\operatorname{S2S}$) error to evaluate \gls{atoms}  \cite{xu_equivariant_2024}. Specifically,  $\operatorname{S2T} = \frac{1}{P}\sum_{p=1}^P\|\hat{\mathbf{x}}_p - \mathbf{x}_p\|_2^2$, measures the average discrepancy between the predicted \(\hat{\mathbf{x}}\) and ground-truth positions \(\mathbf{x}\) across entire trajectories, while  $\operatorname{S2S} = \|\hat{\mathbf{x}}_P - \mathbf{x}_P\|_2^2$, quantifies the error at the final predicted timestep.

\textbf{Baselines.} For comparison, we include a range of classic to state-of-the-art baselines, including Radial Field (RF) \cite{kohler_equivariant_2019}, Tensor Field Networks (TFN) \citep{thomas_tensor_2018}, SE(3) Transformer (SE(3)-Tr.) \citep{fuchs_se3-transformers_2020}, E(n) equivariant graph neural networks (EGNN) \citep{garcia_satorras_en_2021}, MACE \cite{batatia_mace_2023}, and \gls{egno} \citep{xu_equivariant_2024}. We note that MACE is pretrained on the authors’ own dataset, so this is not a strictly like-for-like comparison. Our EGNN baselines are {EGNN-R}ollout (EGNN-R), which predicts timesteps autoregressively, and {EGNN-S}equential (EGNN-S), which uses the output of each \gls{gnn} as the prediction of a given frame. 
We set all baseline hyperparameters following previous works \cite{xu_equivariant_2024, xu_geodiff_2022, shi_learning_2021} and tune \gls{atoms} and \gls{egno} hyperparameters as in \Cref{tab:atom_hparams} and \Cref{tab:egno_hparams}.

\textbf{Training setups.}
For training of \gls{atoms} and \gls{egno}, we consider two temporal discretization strategies in selecting the timestamps $t_p =t+\sum_{r=1}^{p}\Delta t_r$: (1) \textit{Uniform discretization} selects $t_p = t + p/P \Delta T$ and (2) \textit{Tail discretization} selects $t_p = t + \overline{\Delta} + p/P (\Delta T - \overline{\Delta})$ for a lag $\overline{\Delta} \in [0, \Delta T]$. In the main paper, we present experiment results with {uniform discretization} and include the results with tail discretization in Appendix \ref{app:further_experiment}.  We perform early stopping on the lowest S2S validation loss checkpoint and report results as mean \(\pm 2 \sigma\) over \textit{three training runs}. All experiments are run on an NVIDIA\textsuperscript{\textregistered} RTX 5080 with wall-clock time and FLOP utilization detailed in \Cref{tab:md17_compute_cost}.

\subsection{Single-task Learning}
\label{sec:singletask_learning}
We benchmark on the MD17, RMD17, {and MD22} \gls{dft} \gls{md} trajectory datasets \cite{chmiela_machine_2017, christensen_role_2020, chmiela_accurate_2023}. We partition the trajectories into train/validation/test splits of sizes 500/2000/2000, set \(\Delta T = 3000\) fs and \(P = 8\), and train for 2500 epochs following \cite{xu_equivariant_2024}. For the performance on MD17 ( \Cref{tab:md17_results}), we directly quote the results from \cite{xu_equivariant_2024} except for EGNO. 
We design \gls{atoms} to have six transformer blocks with a hidden size of 256.

\textbf{MD17 and RMD17.} 
As shown in \Cref{tab:md17_results}, \gls{atoms} compares favorably with \gls{sota} baselines on MD17 dataset, 
yielding average reductions of 14.96\% (S2S \gls{mse}) and 8.3\% (S2T \gls{mse}) on average\footnote{We exclude benzene from the table due to the previously discussed high numerical noise.}. In \Cref{tab:rmd17_results} (Appendix~\ref{sect:rmd17}), we benchmark \gls{atoms} on RMD17, and observe similarly competitive performance against \gls{egno}.

\begin{wrapfigure}{r}{0.25\textwidth}
    \vspace{-1em}
    \centering
    \includegraphics[width=\linewidth]{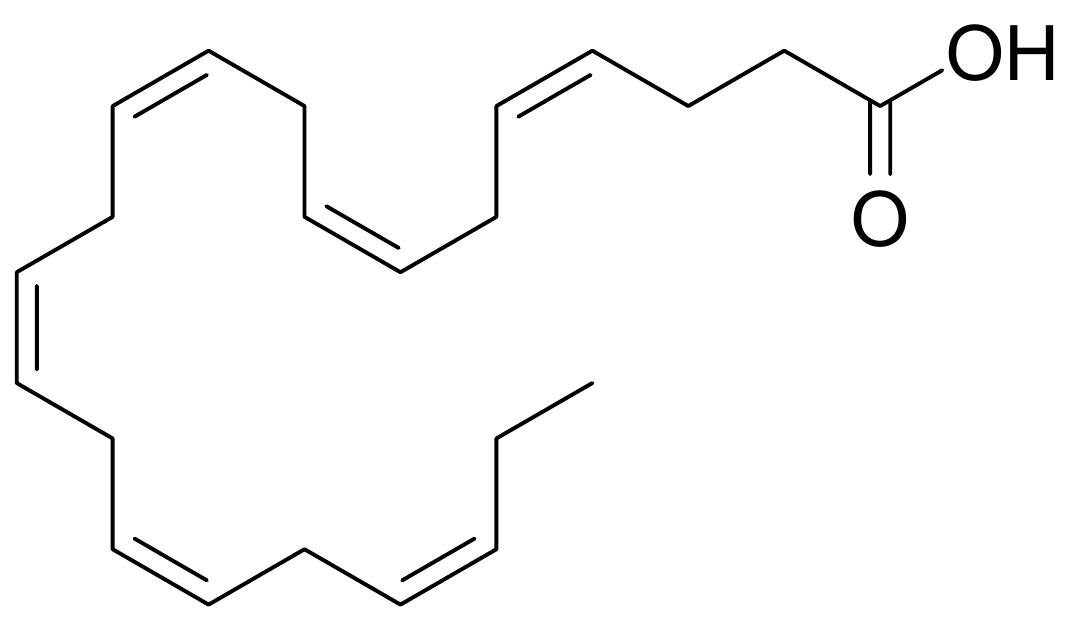}
    \caption{Docosahexaenoic acid (DHA)}
    \label{fig:dha}
    \vspace{-1em}
\end{wrapfigure}
\textbf{MD22.} To evaluate performance on larger molecules, we consider Ac-Ala3-NHMe (20 heavy atoms), docosahexaenoic acid (DHA with 24 heavy atoms), and stachyose (45 heavy atoms) from the MD22 dataset \cite{chmiela_accurate_2023}. \gls{atoms} remains competitive on these systems; whereas \gls{egno} fails to converge (\Cref{tab:md22_results}). We attribute this discrepancy to differing inductive biases: \Glspl{gnn} such as \gls{egno} restrict message passing to a predefined bond or radius graph and can therefore under-represent long-range, non-bonded steric and electrostatic interactions that dominate the behavior of large, sparsely connected molecules \cite{alon_bottleneck_2021,kosmala_ewald-based_2023}. This explains the poor performance of \gls{egno} on MD22, which contains prototypically sparse molecules such as DHA, shown in \Cref{fig:dha} \cite{nv_polyunsaturated_2003}. {We further disentangle the role of connectivity from the use of attention by training a variant, \gls{atoms}-GATv2, in which our heterogeneous temporal attention is replaced by \gls{gatv2} layers \cite{brody_how_2022} operating on the same bond/radius graph as \gls{egno}. \gls{atoms}-\gls{gatv2} still substantially underperforms the full \gls{atoms} model, indicating that the performance gains stem from the fully connected point-cloud interaction pattern rather than from attention alone.}

\begin{table}[t]
\centering
\caption{Single-task MSE (\(\times 10^{-2}\)) on MD17. Upper part: S2S \gls{mse}. Lower part: S2T \gls{mse}.}
\label{tab:md17_results}
\resizebox{\linewidth}{!}{%
    \begin{tabular}{lccccccc}
        \toprule
         & Aspirin & Ethanol & Malonaldehyde & Naphthalene & Salicylic & Toluene & Uracil \\
        \midrule
        RF & \(10.94{\scriptstyle \pm0.02}\) & \(4.64{\scriptstyle \pm0.02}\) & \(13.93{\scriptstyle \pm0.06}\) & \(0.50{\scriptstyle \pm0.02}\) & \(1.23{\scriptstyle \pm0.04}\) & \(10.93{\scriptstyle \pm0.08}\) & \(0.64{\scriptstyle \pm0.02}\) \\
        TFN  & \(12.37{\scriptstyle \pm0.36}\) & \(4.81{\scriptstyle \pm0.08}\) & \(13.62{\scriptstyle \pm0.16}\) & \(0.49{\scriptstyle \pm0.02}\) & \(1.03{\scriptstyle \pm0.04}\) & \(10.89{\scriptstyle \pm0.02}\) & \(0.84{\scriptstyle \pm0.04}\) \\
        SE(3)-Tr. & \(11.12{\scriptstyle \pm0.12}\) & \(4.74{\scriptstyle \pm0.02}\) & \(13.89{\scriptstyle \pm0.04}\) & \(0.52{\scriptstyle \pm0.02}\) & \(1.13{\scriptstyle \pm0.04}\) & \(10.88{\scriptstyle \pm0.12}\) & \(0.79{\scriptstyle \pm0.04}\) \\
        EGNN & \(14.41{\scriptstyle \pm0.30}\) & \(4.64{\scriptstyle \pm0.04}\) & \(13.64{\scriptstyle \pm0.02}\) & \(0.47{\scriptstyle \pm0.04}\) & \(1.02{\scriptstyle \pm0.04}\) & \(11.78{\scriptstyle \pm0.14}\) & \(0.64{\scriptstyle \pm0.02}\) \\
        EGNN-R & \(9.96{\scriptstyle \pm0.14}\) & \(4.61{\scriptstyle \pm0.01}\) & \(13.04{\scriptstyle \pm0.03}\) & \(0.44{\scriptstyle \pm0.05}\) & \(0.96{\scriptstyle \pm0.01}\) & \(10.19{\scriptstyle \pm0.15}\) & \(1.11{\scriptstyle \pm0.04}\) \\
        EGNN-S & \(10.25{\scriptstyle \pm0.09}\) & \(4.61{\scriptstyle \pm0.01}\) & \(13.06{\scriptstyle \pm0.01}\) & \(0.53{\scriptstyle \pm0.01}\) & \(1.06{\scriptstyle \pm0.05}\) & \(10.83{\scriptstyle \pm0.09}\) & \(0.62{\scriptstyle \pm0.01}\) \\
        \gls{egno} & \(9.64{\scriptstyle \pm0.15}\) & \(4.57{\scriptstyle \pm0.01}\) & \(\mathbf{12.92}{\scriptstyle \pm0.00}\) & \(\mathbf{0.39}{\scriptstyle \pm0.00}\) & \(0.89{\scriptstyle \pm0.01}\) & \(11.00{\scriptstyle \pm0.00}\) & \(\mathbf{0.58}{\scriptstyle \pm0.02}\) \\
          MACE & \(6.95\pm\scriptstyle{0.00}\) & \(2.06\pm\scriptstyle{0.00}\) & \(17.99\pm\scriptstyle{0.26}\) & \(0.72\pm\scriptstyle{0.00}\) & \(1.05\pm\scriptstyle{0.00}\) & \(6.44\pm\scriptstyle{0.00}\) & \(0.75\pm\scriptstyle{0.00}\) \\  
        \textbf{\gls{atoms}} & \(\mathbf{6.82}{\scriptstyle \pm0.06}\) & \(\mathbf{3.52}{\scriptstyle \pm0.04}\) & \(14.72{\scriptstyle \pm0.01}\) & \(0.50{\scriptstyle \pm0.00}\) & \(\mathbf{0.88}{\scriptstyle \pm0.01}\) & \(\mathbf{4.66}{\scriptstyle \pm0.21}\) & \(0.63{\scriptstyle \pm0.00}\) \\
        \midrule
        EGNN-R & \(7.35{\scriptstyle \pm0.19}\) & \(3.21{\scriptstyle \pm0.00}\) & \(\mathbf{10.75}{\scriptstyle \pm0.04}\) & \(\mathbf{0.34}{\scriptstyle \pm0.06}\) & \(1.09{\scriptstyle \pm0.12}\) & \(4.53{\scriptstyle \pm0.08}\) & \(0.89{\scriptstyle \pm0.02}\) \\
        EGNN-S & \(9.01{\scriptstyle \pm0.34}\) & \(3.21{\scriptstyle \pm0.00}\) & \(11.20{\scriptstyle \pm0.03}\) & \(0.42{\scriptstyle \pm0.01}\) & \(1.41{\scriptstyle \pm0.00}\) & \(4.86{\scriptstyle \pm0.04}\) & \(0.65{\scriptstyle \pm0.01}\) \\
        \gls{egno} & \(9.64{\scriptstyle \pm0.15}\) & \(4.57{\scriptstyle \pm0.01}\) & \(12.92{\scriptstyle \pm0.00}\) & \(0.39{\scriptstyle \pm0.00}\) & \(0.90{\scriptstyle \pm0.01}\) & \(10.99{\scriptstyle \pm0.00}\) & \(\mathbf{0.58}{\scriptstyle \pm0.02}\) \\
         MACE & \(5.06\pm\scriptstyle{0.00}\) & \(2.84\pm\scriptstyle{0.00}\) & \(16.09\pm\scriptstyle{0.03}\) & \(0.57\pm\scriptstyle{0.00}\) & \(0.55\pm\scriptstyle{0.00}\) & \(3.26\pm\scriptstyle{0.00}\) & \(1.08\pm\scriptstyle{0.00}\) \\
        \textbf{\gls{atoms}} & \(\mathbf{5.62}{\scriptstyle \pm0.05}\) & \(\mathbf{2.62}{\scriptstyle \pm0.04}\) & \(12.49{\scriptstyle \pm0.01}\) & \(0.43{\scriptstyle \pm0.00}\) & \(\mathbf{0.86}{\scriptstyle \pm0.01}\) & \(\mathbf{2.27}{\scriptstyle \pm0.10}\) & \(0.61{\scriptstyle \pm0.00}\) \\
        \bottomrule
    \end{tabular}
}
\end{table}



\subsection{Multitask Learning on TG80}
\label{sec:multitask_learning}

We pretrain \gls{atoms} on TG80, scaling to six attention blocks with a hidden size of 256. We select stochastic horizons \(\Delta T \sim \operatorname{LogUnif}(8 \text{ fs}, 24\,000 \text{ fs})\) and use a five-fold, cluster-based cross-validation. Specifically, we compute ECFP-4 fingerprints \cite{rogers_extended-connectivity_2010}, embed them using UMAP \cite{mcinnes_umap_2018}, and apply agglomerative clustering \cite{ward_hierarchical_1963} to partition compounds into ten disjoint clusters. The folds are then formed by holding out clusters, ensuring that the train/validation/test sets occupy distinct regions of chemical space. This cluster-wise protocol minimizes leakage and more closely reflects the prospective scientific setting in which models must generalize to unseen molecules. Cluster-based approaches present more challenging generalization problems than random splits or common chemical-scaffold-based splits \cite{guo_scaffold_2024}. In Appendix \ref{sect:random_split}, we also consider pretraining on a standard random split of molecules. 

Table \ref{tab:tg80_umap} benchmarks \gls{atoms} by assessing both \textit{in-distribution} (ID) and \textit{out-of-domain} (OOD) S2T performance. For the \textit{in-distribution} setting, we train, validate, and test on molecules from the same cluster. We observe that \gls{atoms} outperforms existing baselines by an average of 83.96\% in terms of S2T \gls{mse}. We then assess \textit{out-of-domain} (OOD) generalization performance by predicting the dynamics of unseen compounds drawn from disjoint clusters. Under OOD settings, \gls{atoms} nearly halves the S2T \gls{mse} of \gls{egno}, with an average improvement of $39.74\%$ across five cluster splits. Notably, OOD \gls{atoms} beats ID EGNO performance in four of five folds. This striking zero-shot generalization, realized without any exposure to the test molecules, confirms that \gls{atoms} uniquely learns robust, transferable knowledge of molecular dynamics. In Appendix \ref{app:s2s_result}, we show similar outperformance in S2S prediction. In Appendix \ref{sec:model_compute_time}, we show that the significantly improved multitask performance comes with a modest overhead in training time and in inference latency.

\begin{figure}[t]
  \centering
  \begin{minipage}[t]{0.41\textwidth}
    \centering
    \captionof{table}{{Single-task \gls{mse} (\(\times 10^{-2}\)) on MD22. Upper: S2S. Lower: S2T}}
    \label{tab:md22_results}
    \vspace{-0.8em}
    \resizebox{\linewidth}{!}{%
    \begin{tabular}{lccc}
        \toprule
         & \textbf{Ac-Ala3-NHME} & \textbf{DHA} & \textbf{Stachyose} \\
        \midrule
        \gls{egno} & \(357.89{\scriptstyle \pm3.94}\) & \(178.39{\scriptstyle \pm4.91}\) & \(42.11{\scriptstyle \pm0.10}\) \\
          \gls{atoms}-\gls{gatv2} & \(223.57{\scriptstyle \pm0.66}\) & \(16.72{\scriptstyle \pm0.44}\) & \(41.40{\scriptstyle \pm0.37}\) \\
        \textbf{\gls{atoms}} & \(\mathbf{9.65}{\scriptstyle \pm0.75}\) & \(\mathbf{10.60}{\scriptstyle \pm1.11}\) & \(\mathbf{21.25}{\scriptstyle \pm4.20}\) \\
        \midrule
        \rowcolor{gray!20} Gap & \(+97.30\%\) & \(+94.06\%\) & \(+49.54\%\) \\
        \midrule
        \gls{egno} & \(232.40{\scriptstyle \pm6.75}\) & \(116.45{\scriptstyle \pm3.34}\) & \(30.84{\scriptstyle \pm0.03}\) \\
        \gls{atoms}-\gls{gatv2} & \(113.26{\scriptstyle \pm0.04}\) & \(14.39{\scriptstyle \pm0.32}\) & \(29.70{\scriptstyle \pm0.15}\) \\
        \textbf{\gls{atoms}} & \(\mathbf{7.55}{\scriptstyle \pm0.42}\) & \(\mathbf{9.66}{\scriptstyle \pm1.16}\) & \(\mathbf{18.13}{\scriptstyle \pm3.78}\) \\
        \midrule
        \rowcolor{gray!20} Gap & \(+96.75\%\) & \(+91.70\%\) & \(+41.22\%\) \\
        \bottomrule
    \end{tabular}
        }
  \end{minipage}
  \hfill
    \begin{minipage}[t]{0.57\textwidth}
        \centering
          \captionof{table}{Multitask S2T \gls{mse} \((\times10^{-2})\) on TG80 across five UMAP cluster assignments.}
          \label{tab:tg80_umap}
          \vspace{-0.8em}
          \resizebox{\linewidth}{!}{%
          \begin{tabular}{llccccc}
            \toprule
             &  & \textbf{Cluster 1} & \textbf{Cluster 2} & \textbf{Cluster 3} & \textbf{Cluster 4} & \textbf{Cluster 5} \\
            \midrule
            \multirow{3}{*}{ID}
              & \gls{egno}  & \(44.23{\scriptstyle \pm0.68}\) & \(95.52{\scriptstyle \pm0.73}\) & \(141.16{\scriptstyle \pm0.21}\) & \(150.92{\scriptstyle \pm0.11}\) & \(107.47{\scriptstyle \pm0.36}\) \\
              & \textbf{\gls{atoms}} & \(\mathbf{9.71}{\scriptstyle \pm0.75}\) & \(\mathbf{18.26}{\scriptstyle \pm1.58}\) & \(\mathbf{16.82}{\scriptstyle \pm1.46}\) & \(\mathbf{16.93}{\scriptstyle \pm3.65}\) & \(\mathbf{17.20}{\scriptstyle \pm0.46}\) \\ \cmidrule(lr){2-7}
              \rowcolor{gray!20} & Gap & \(78.04\%\) & \(80.89\%\) & \(88.09\%\) & \(88.78\%\) & \(83.99\%\) \\
              \midrule
              \multirow{6}{*}{\cellcolor{white} OOD}
               & MACE & \(134.26\) & \(224.12\) & \(325.97\) & \(316.26\) & \(229.64\) \\
              & \gls{egno}  & \(45.95{\scriptstyle \pm0.80}\) & \(115.43{\scriptstyle \pm13.23}\) & \(151.74{\scriptstyle \pm0.57}\) & \(163.90{\scriptstyle \pm0.69}\) & \(113.68{\scriptstyle \pm2.50}\) \\
              & EGNN\text{-}S & \(45.44{\scriptstyle \pm0.57}\) & \(7386.15{\scriptstyle \pm6931.89}\) & \(152.72{\scriptstyle \pm0.83}\) & \(464.22{\scriptstyle \pm509.48}\) & \(114.30{\scriptstyle \pm0.79}\) \\
              & EGNN\text{-}R & \(44.88{\scriptstyle \pm0.68}\) & \(109.62{\scriptstyle \pm1.92}\) & \(148.05{\scriptstyle \pm0.70}\) & \(161.54{\scriptstyle \pm0.68}\) & \(110.10{\scriptstyle \pm0.96}\) \\
              & \textbf{\gls{atoms}} & \(\mathbf{35.05}{\scriptstyle \pm0.97}\) & \(\mathbf{106.99}{\scriptstyle \pm104.64}\) & \(\mathbf{60.95}{\scriptstyle \pm4.86}\) & \(\mathbf{66.68}{\scriptstyle \pm0.96}\) & \(\mathbf{47.49}{\scriptstyle \pm1.59}\) \\ 
              \cmidrule(lr){2-7}
             \rowcolor{gray!20} & Gap & \(21.93\%\) & \(2.40\%\) & \(58.83\%\) & \(58.71\%\) & \(56.88\%\) \\
            \bottomrule
          \end{tabular}
          }

    \end{minipage}
\end{figure}

\subsection{Temporal Gap and Timestep Invariance Properties}
\textbf{\(\Delta T\) Invariance.}
We evaluate the performance of {pretrained} \gls{atoms} {(at fixed $\Delta T =3000$)} with varying $\Delta T$ {at inference}. We compare \gls{atoms}, \gls{egno}, and \gls{egnn} on S2T \gls{mse} by fixing \(P=8\) and sweeping \(\Delta T\) logarithmically from {10} to 10\,000 {fs} on an in-distribution (Cluster 1) multitask model. In \Cref{fig:delta_t_invariance}, we show that \gls{atoms} maintains its extrapolation advantage across the range compared to \gls{egno}, particularly at larger $\Delta T$. Ablating T-RoPE (NoPE) removes this advantage by exhibiting an EGNO-like error trend with substantially higher MSE. This 
underscores T-RoPE's role in stable time-gap extrapolation.



\textbf{\(P\) Invariance.} 
Following the discretization invariance in neural operators, we expect \gls{atoms} and \gls{egno} models to show consistent \gls{mse} as \(P\) varies under uniform discretization \cite{kovachki_neural_2021}.  \Cref{fig:p_invariance} confirms such a conjecture by showing that multitask \gls{atoms} pretrained at \(P=8\) maintain constant S2T \gls{mse} as \(P\) ranges from 4 to 24 {at inference}.


\begin{figure}[H]
\centering
\hfill
\begin{minipage}[t]{0.45\linewidth}
    \centering
    \includegraphics[width=\linewidth]{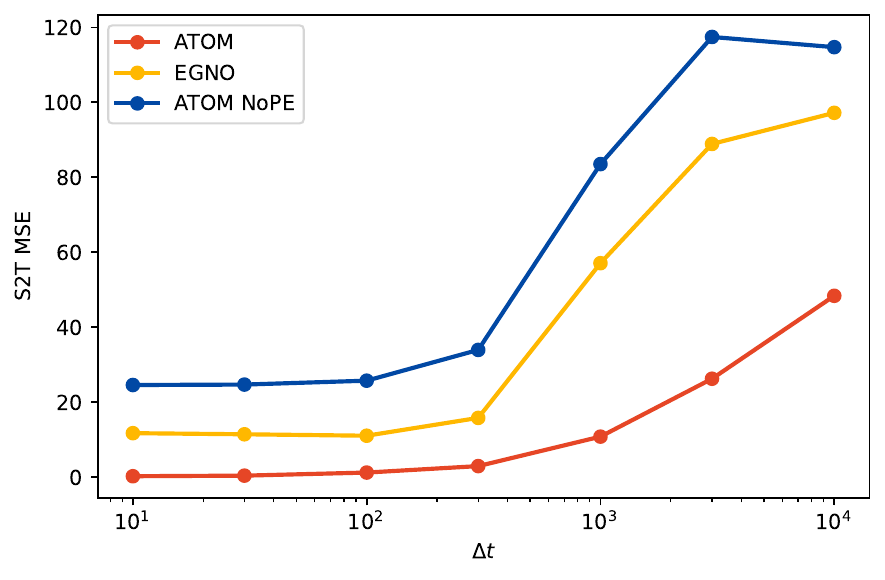}
    \caption{Pretrained ({\(\Delta T = 3000\)}, ID) multitask S2T \gls{mse} across varying \(\Delta T\) values.}
    \label{fig:delta_t_invariance}
\end{minipage}
\hfill
\begin{minipage}[t]{0.45\linewidth}
    \centering
    \includegraphics[width=\linewidth]{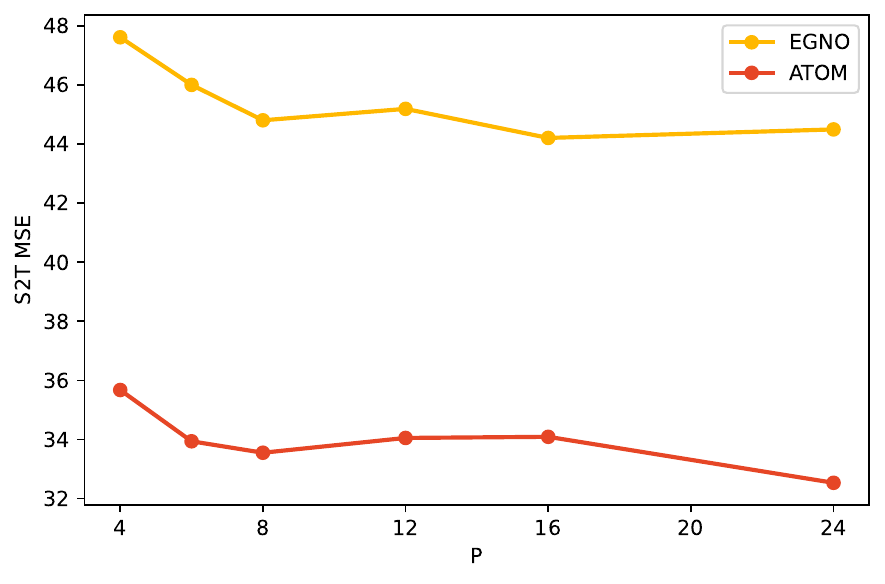}
    \caption{{Pretrained (\(P=8\))} \gls{atoms} and \gls{egno} are discretization invariant, showing stable S2T \gls{mse}.}
    \label{fig:p_invariance}
\end{minipage}
\hfill
\label{fig:invariances}
\end{figure}

\subsection{Ablation studies}
\label{sec:ablations}

We perform extensive ablations to assess each design choice in \gls{atoms}. For single-task performance (Fig.~\ref{fig:MD17_ablation}) and multitask performance (Fig.~\ref{fig:TG80_ablation}), we independently toggle components and measure their contributions. Our analysis focuses on equivariant lifting, T-RoPE, label-noise regularization, heterogeneous attention, and random-walk positional encoding (under multitask pretraining).

\FPeval{\NoEquivariantLiftingvsDefault}{round(29.40 - 6.92, 2)}
\FPeval{\FullyEquivariantvsDefault}{round(10.25 - 6.77, 2)}
\FPeval{\AtomRotationPenalty}{round(73.04 / 6.76, 2)}
\FPeval{\NoLiftRotationPenalty}{round(660.97 / 33.44, 2)}

\begin{wrapfigure}{r}{0.48\textwidth}
    \centering
    \setlength{\extrarowheight}{1pt}   
    \vspace{-10pt}
    \captionof{table}{S2T MSE (\(\times 10^{-2}\)) of a fixed input frame rotated and unrotated by an \(\mathrm{SO}(3)\) matrix.}
    \resizebox{\linewidth}{!}{%
    \begin{tabular}{l cc}
        \toprule
        & ATOM & No equivariant Lift \\
        \midrule
        Unrotated & \(6.76{\scriptstyle \pm0.69}\) & \(33.44{\scriptstyle \pm23.42}\) \\
        Rotated & \(73.04{\scriptstyle \pm27.01}\) & \(660.97{\scriptstyle \pm945.86}\) \\
        \midrule
        Increase & \(66.28{\scriptstyle \pm26.32}\) & \(627.53{\scriptstyle \pm926.53}\) \\
        \rowcolor{gray!20} Rotation penalty (\(\times\)) & \(\AtomRotationPenalty\) & \(\NoLiftRotationPenalty\) \\
        \bottomrule
    \end{tabular}
    }
    \label{tab:equivariance_error}
\end{wrapfigure}

\textbf{Equivariant lifting.} We assess the quasi-equivariant design against a non-equivariant ATOM. As shown in \Cref{fig:MD17_ablation}, replacing the equivariant lifting introduced in \Cref{sec:equivariant_lifting} with standard linear layers (no equivariant lifting) markedly degrades the performance of \gls{atoms}, increasing S2T \gls{mse} by \NoEquivariantLiftingvsDefault{}. \Cref{tab:equivariance_error} further quantifies sensitivity to \(\mathrm{SO}(3)\) rotations: ATOM's S2T \gls{mse} increases by \(\AtomRotationPenalty\times\) under rotation, compared to \(\NoLiftRotationPenalty\times\) without equivariant lifting. Notably, the fully-equivariant variant of \gls{atoms}, described in appendix \Cref{sec:fully_equivariant}, also underperforms \gls{atoms} in both single-task (\Cref{fig:MD17_ablation}) and multitask (\Cref{fig:TG80_ablation}) settings, with the gap exaggerated in the multitask setting. This aligns with recent findings on relaxed equivariance, suggesting that strict equivariance can limit model capacity and complicate the optimization process \cite{elhag_relaxed_2025}. We present estimates of the quasi-equivariance \(\varepsilon\) in \Cref{app:quasi_equivariance_mc}.

\FPeval{\CrossAttvsStdAtt}{round(7.39 - 6.92, 2)}
\textbf{Heterogeneous attention.} We find that substituting heterogeneous temporal attention with standard self-attention on the phase space features increases S2S MSE by \CrossAttvsStdAtt{}, suggesting that cross-attention enables access to non-trivial feature interactions.


\FPeval{\NoPEvsDefault}{round(5.262 - 4.197, 2)}
\textbf{\gls{trope}.} 
In the single-task regime with fixed $\Delta T$ (\Cref{fig:MD17_ablation}), \gls{trope} contributes little to the performance of \gls{atoms}, as it effectively reduces to a constant rotational shift. By contrast, with stochastic $\Delta T$,  disabling \gls{trope} (NoPE) increases \gls{mse} by \NoPEvsDefault{}, consistent with \gls{atoms} leveraging the \(\tau\) parameter to encode  variable time gaps (\Cref{fig:TG80_ablation}).  An \gls{egno}-style sinusoidal positional encoding produces a similar performance degradation.


\textbf{Label noise regularization.}
\FPeval{\NoNoisevsDefault}{round(8.13 - 6.92, 2)}
We also test the utility of label noise regularization as in \Cref{sec:label_noise}. From \Cref{fig:MD17_ablation}, we observe that removing augmented noise from the position and velocity features increased S2T \gls{mse} by \NoNoisevsDefault{}. For the multitask ablation on TG80, we suppress label noise regularization, as the dataset is designed to be numerically stable with small noise.

\textbf{RWPE.} We assess random-walk positional encoding (RWPE) in the multitask pretraining. Figure~\ref{fig:TG80_ablation} indicates that RWPE facilitates molecule identification, yielding improved multitask performance.


\begin{figure}[t]
\centering
\begin{minipage}[t]{0.46\linewidth}
    \centering
    \includegraphics[width=\linewidth]{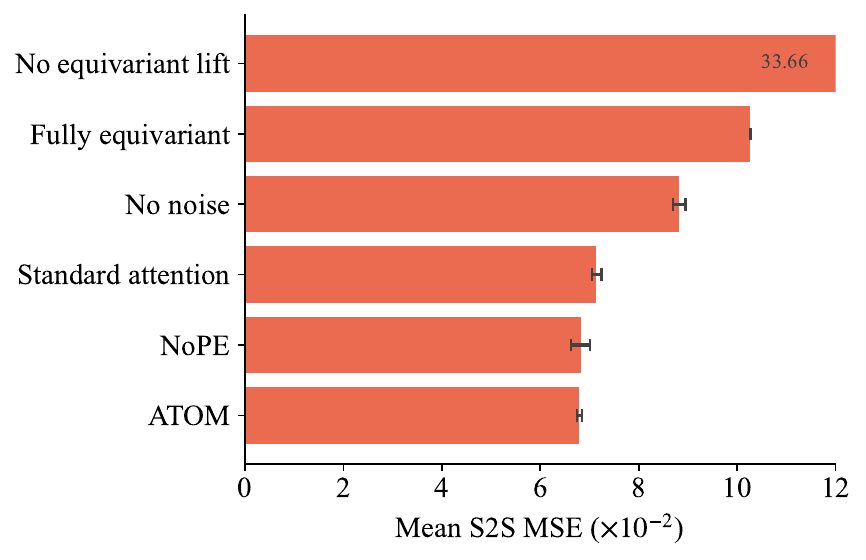}
    \caption{\gls{atoms} ablation on MD17 Aspirin.}
    \label{fig:MD17_ablation}
\end{minipage}
\hfill
\begin{minipage}[t]{0.46\linewidth}
    \centering
    \includegraphics[width=\linewidth]{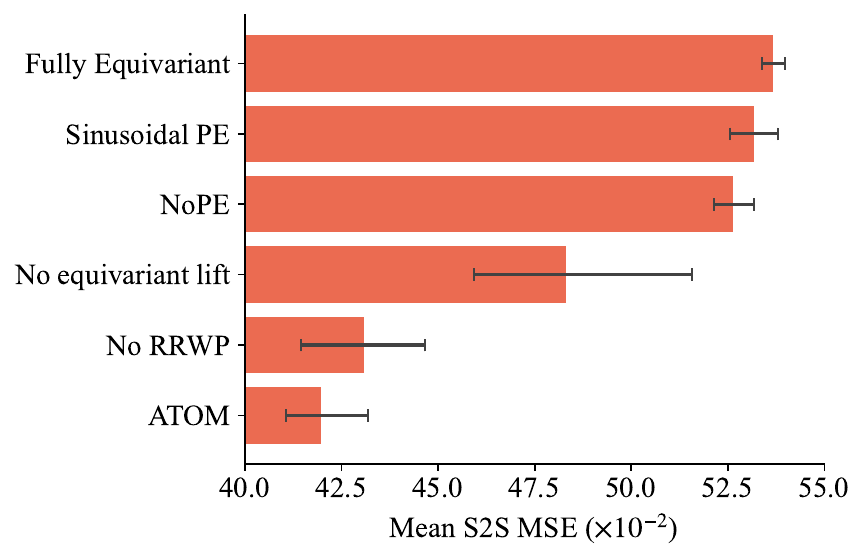}
    \caption{\gls{atoms} ablation on TG80 Cluster 1.}
    \label{fig:TG80_ablation}
\end{minipage}
\label{fig:ablations}
\end{figure}

\section{Conclusions}

In this work, we demonstrate that carefully designed transformer neural operators enable zero-shot generalization to unseen chemical dynamics. Our experiments on MD17 demonstrate continued good single-task performance, and we present the first molecular neural operator that can successfully learn large molecule dynamics using MD22. Our multitask experiments show that our method learns transferable dynamics knowledge, even without explicit graph representations. In combination with our TG80 dataset, we provide a large-scale open-source benchmark and baselines to evaluate future models and spur further operator research with concrete scientific applicability.

\paragraph{Limitations} We remark that TG80 does not contain trajectories for large molecules with more than 15 heavy atoms, despite their obvious chemical and pharmacological relevance. In follow-up work, we intend to enrich TG80 with such molecules, calculated with a higher-resolution DFT basis set, \(\omega\)B97X-3c \cite{muller_omegab97x-3c_2023}. Regarding \gls{atoms}, it lacks an explicit energy-based inductive bias, which may permit long-horizon drift. A natural extension is therefore a framewise energy head \(E_\theta(\mathbf{x}_{t_p})\) with force supervision \(\mathbf{F}_{t_p}=-\nabla_{\mathbf{x}_{t_p}}E_\theta(\mathbf{x}_{t_p})\). {This energy term also defines the drift in the Langevin dynamics, where an additional stochastic term accounts for thermal perturbations of atomic positions. Incorporating such physics-informed stochastic dynamics into our operator learning framework is a natural next step, and we view this as a promising direction for future \gls{md} research.}




\section*{Reproducibility Statement}
We provide experiment details, such as choice of hyperparameters and other training configurations in Appendix \ref{sect:details}. In addition, we will release the TG80 dataset upon acceptance under MIT license for reproducibility.
\newpage

\bibliographystyle{ICLR2026/iclr2026_conference}
\bibliography{library}

\appendix
\clearpage

\renewcommand{\contentsname}{\textsc{Appendices}}
\addcontentsline{toc}{section}{Appendices}  
\addtocontents{toc}{\protect\setcounter{tocdepth}{4}} 
\tableofcontents

\clearpage
\section{LLM Acknowledgment}

The authors acknowledge the use of LLM for grammar corrections and for improving the clarity of the manuscript. The LLM was not used for generating original ideas, content, or experimental results. All conceptual contributions, analyses, and conclusions presented in this work are entirely from the authors.

\section{Background}

This section provides an introduction to the preliminaries of group theory. 

\subsection{Groups}
\label{sec:groups}
A group \( (G, \circ) \) consists of a non-empty set \( G \) and a binary operation \(\circ: G \times G \to G\) satisfying the following axioms:
\begin{enumerate}
    \item \textbf{Closure:} For all \( a, b \in G \), the result of the operation \( a \circ b \) is also in \( G \): \( a \circ b \in G \).
    \item \textbf{Identity Element:} There exists an element \( e \in G \) such that, for all \( a \in G \), \( a \circ e = e \circ a = a \).
    \item \textbf{Associativity:} For all \( a, b, c \in G \), \( (a \circ b) \circ c = a \circ (b \circ c) \).
    \item \textbf{Inverses:} For each \( a \in G \), there exists an element \( a^{-1} \in G \) such that \( a \circ a^{-1} = a^{-1} \circ a = e \).
\end{enumerate}

In general, not all groups are abelian. That is, the binary operation \(\circ\) does not necessarily commute:  
$
g \circ h = h \circ g,  \forall g,h \in G.
$

\subsection{Group Representations}
\label{sec:group_representations}
A group representation is a homomorphism  \(\rho: G \to GL(V)\) that assigns an \(n \times n\) matrix to each group element \(g \in G\), realizing it as a linear transformation. Representations must preserve the binary operation for all members of the group \(G\) such that:
\[
\rho(g \circ h) = \rho(g) \rho (h), \quad \forall g,h \in G.
\]

A representation \(\rho(g)\) is reducible if it can be represented as the direct sum of other representations:
\[
\rho(g) = \rho_1(g) \oplus \rho_2(g), \quad \forall g \in G.
\]

For example, a reducible \(4 \times 4\) representation of \(SU(2)\) can be decomposed into two \(2 \times 2\) sub-representations:
\[
\rho(g) =
\begin{bmatrix}
\rho_1(g) & 0 \\
0 & \rho_2(g)
\end{bmatrix}, \quad \forall g \in SU(2),
\]

where \(\rho_1(g)\) and \(\rho_2(g)\) are the following irreducible representations of \(SU(2)\):

\[
\rho_1(g) =
\begin{bmatrix}
e^{i\theta} & 0 \\
0 & e^{-i\theta}
\end{bmatrix}, \quad
\rho_2(g) =
\begin{bmatrix}
e^{i\phi} & 0 \\
0 & e^{-i\phi}
\end{bmatrix}.
\]
By contrast, irreducible representations or \textit{irreps} cannot be represented as such a direct sum. Formally, they have no non-trivial invariant subspaces \(W \subset V\) such that \(\rho(g)W \subset W, \forall g \in G\).

Representing inputs as irreps ensures equivariance by constraining each feature to transform predictably under group actions. Given \( V = \bigoplus_{i} V_i \) with irreps \( V_i \), the transformation of an input \( x \in V \) under \( g \in G \) is:  
\[
\rho(g)x = \bigoplus_{i} \rho_i(g)x_i.
\]  
Each component \( x_i \) transforms independently according to \(\rho_i\), preserving symmetry. Scalars remain invariant, while vectors rotate according to standard representations. This decomposition prevents the mixing of differently transforming features, ensuring that all subsequent operations, linear or non-linear, respect the group’s symmetry, thereby maintaining equivariance throughout the network.



Intuitively, the tensor products capture interactions between features in a manner akin to multiplication, producing a higher-dimensional representation.
Crucially, this new representation is reducible, so we may decompose it into irreps:
\[
V \otimes V \cong \bigoplus_k V_k.
\]
It is this decomposition that allows the network to project onto individual irreps, achieving non-trivial feature mixing whilst preserving symmetry constraints.

\section{Datasets}
\label{sec:dataset_construction}

We present a visualization of a sample trajectory of uracil from three datasets in Figure \ref{fig:uracil_trajectory}.


\begin{figure}[h]
    \centering
    \includegraphics[width=0.8\linewidth]{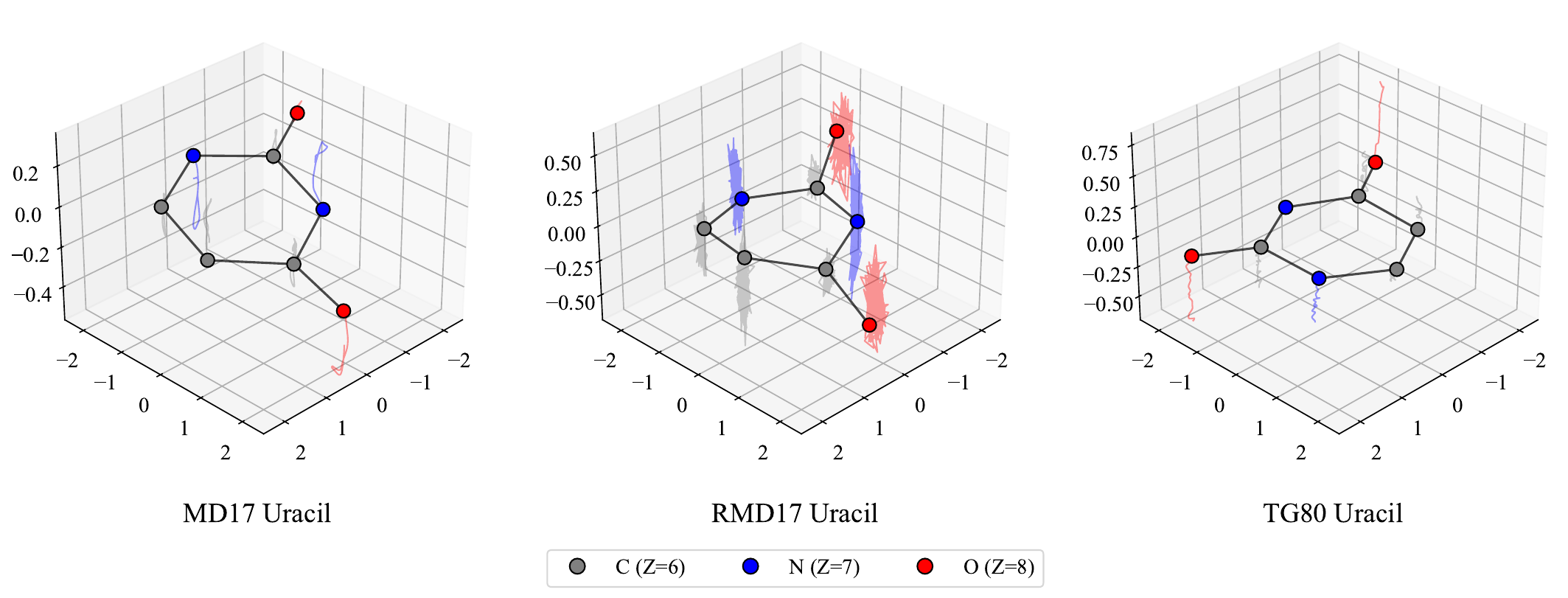}
    \caption{3000 timesteps of uracil trajectory from MD17, RMD17, and TG80.}
    \label{fig:uracil_trajectory}
\end{figure}

\subsection{Licences}
\begin{table}[ht]
\centering
\caption{Dataset sources and licenses. We release TG80 under the MIT license.}
\label{tab:dataset_sources}
\rowcolors{2}{gray!10}{white}
\resizebox{\linewidth}{!}{
    \begin{tabular}{lll}
    \toprule
    \textbf{Dataset} & \textbf{Source} & \textbf{License} \\
    \midrule
    MD17            & \url{https://www.sgdml.org/} & \texttt{CC BY 4.0}\\
    RMD17           & \url{https://archive.materialscloud.org/record/2020.82} & \texttt{CC Zero V1.0 Universal} \\
    MD22           & \url{https://www.sgdml.org/} & \texttt{CC BY 4.0} \\
    TG80            & To be released at \url{URL} & \texttt{MIT}\\
    \bottomrule
    \end{tabular}}
\end{table}

\subsection{Model Inputs and the Dataloader}
\label{sec:dataloader}
Our compound representations follow \cite{shi_learning_2021, xu_equivariant_2024}. We model hydrogen atoms implicitly and concatenate the position and velocity norms for each node \(i\) with their respective vectors. Unlike their implementations, we avoid explicit graph construction and do not include edge labels describing atomic bond geometries.

We duplicate all frames \(\mathcal{G}^{(t)} \to \{\mathcal{G}^{(t)}\}^{P}\) during dataset initialization, producing a five-fold improvement in throughput compared to previous dataloaders in \Cref{tab:dataloader_perf}. 

\begin{table}[ht]
    \centering
    \caption{Mean time (seconds) to produce 10\,000 batches over 100 benchmark runs. Batch size = 100, 500 samples, \(\Delta t = 3000\), 500 warmup batches.}
    \label{tab:dataloader_perf}
    \begin{tabular}{lcccc}
        \toprule
         & Aspirin & Ethanol & Naphthalene & Toluene \\
         \midrule
         \gls{egno} & \(0.060{\scriptstyle \pm0.024}\) & \(0.024{\scriptstyle \pm0.016}\) & \(0.056{\scriptstyle \pm0.024}\) & \(0.039{\scriptstyle \pm0.024}\) \\
         \gls{atoms} & \(\textbf{0.005}{\scriptstyle \pm0.002}\) & \(\textbf{0.007}{\scriptstyle \pm0.002}\) & \(\textbf{0.008}{\scriptstyle \pm0.004}\) & \(\textbf{0.006}{\scriptstyle \pm0.002}\) \\
         \bottomrule
    \end{tabular}
\end{table}

\newpage
\subsection{Numerical Stability}
We evaluate the numerical stability of MD17, RMD17, and TG80. MD17 benzene exhibits substantial center-of-mass drift in \Cref{fig:MD17_numerical_quality}, which is also partially visible in the consistent motion trails shown in \Cref{fig:MD17_trajectories}. RMD17 exhibits improved stability, with no center-of-mass drift exceeding \(1\times10^4\). TG80 shows the lowest drift of all datasets, and expectedly includes more molecules with high per-step drift (due to more complex sterically hindered geometries).
\begin{figure}[H]
    \begin{subfigure}{0.49\textwidth}
        \centering
        \includegraphics[width=\linewidth]{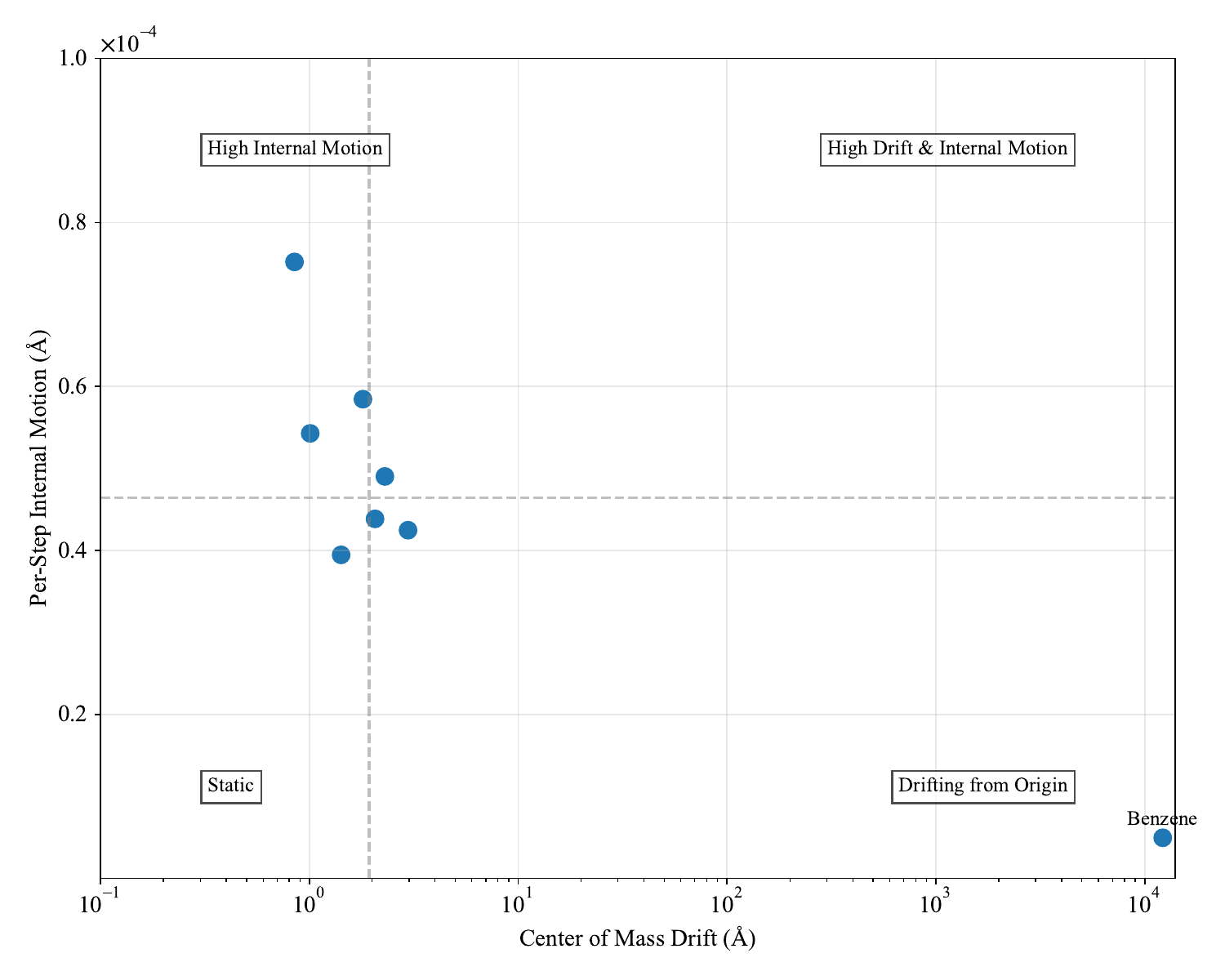}
        \caption{MD17 molecules are largely consistent, except for benzene, which exhibits substantial drift.}
        \label{fig:MD17_numerical_quality}
    \end{subfigure}
    \hspace{0.02\textwidth}
    \begin{subfigure}{0.49\textwidth}
        \centering
        \includegraphics[width=\linewidth]{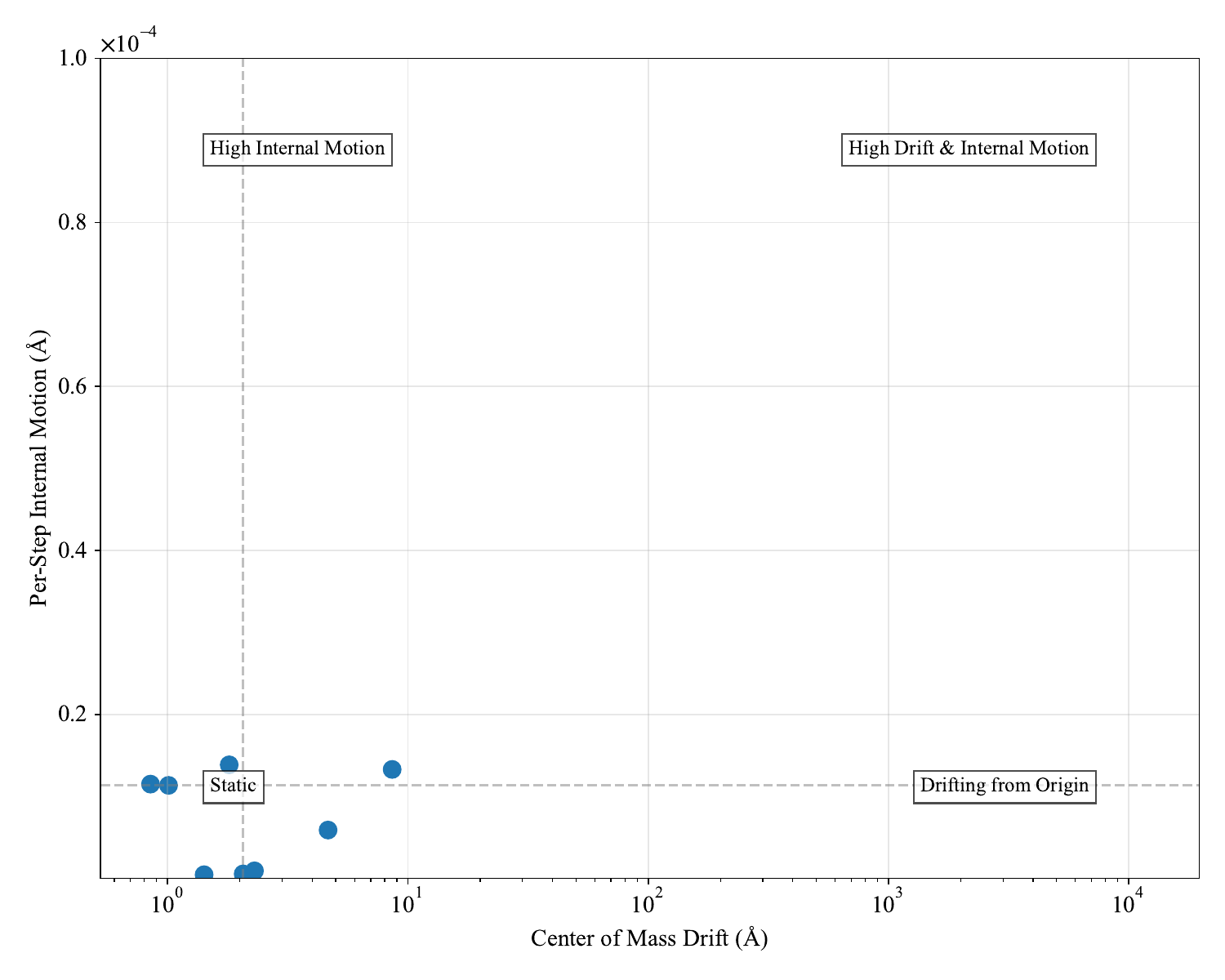}
        \caption{RMD17 molecules are more numerically stable, supporting their use in future benchmarks.}
        \label{fig:RMD17_numerical_quality}
    \end{subfigure}
    \begin{subfigure}{\textwidth}
        \centering
        \includegraphics[width=\linewidth]{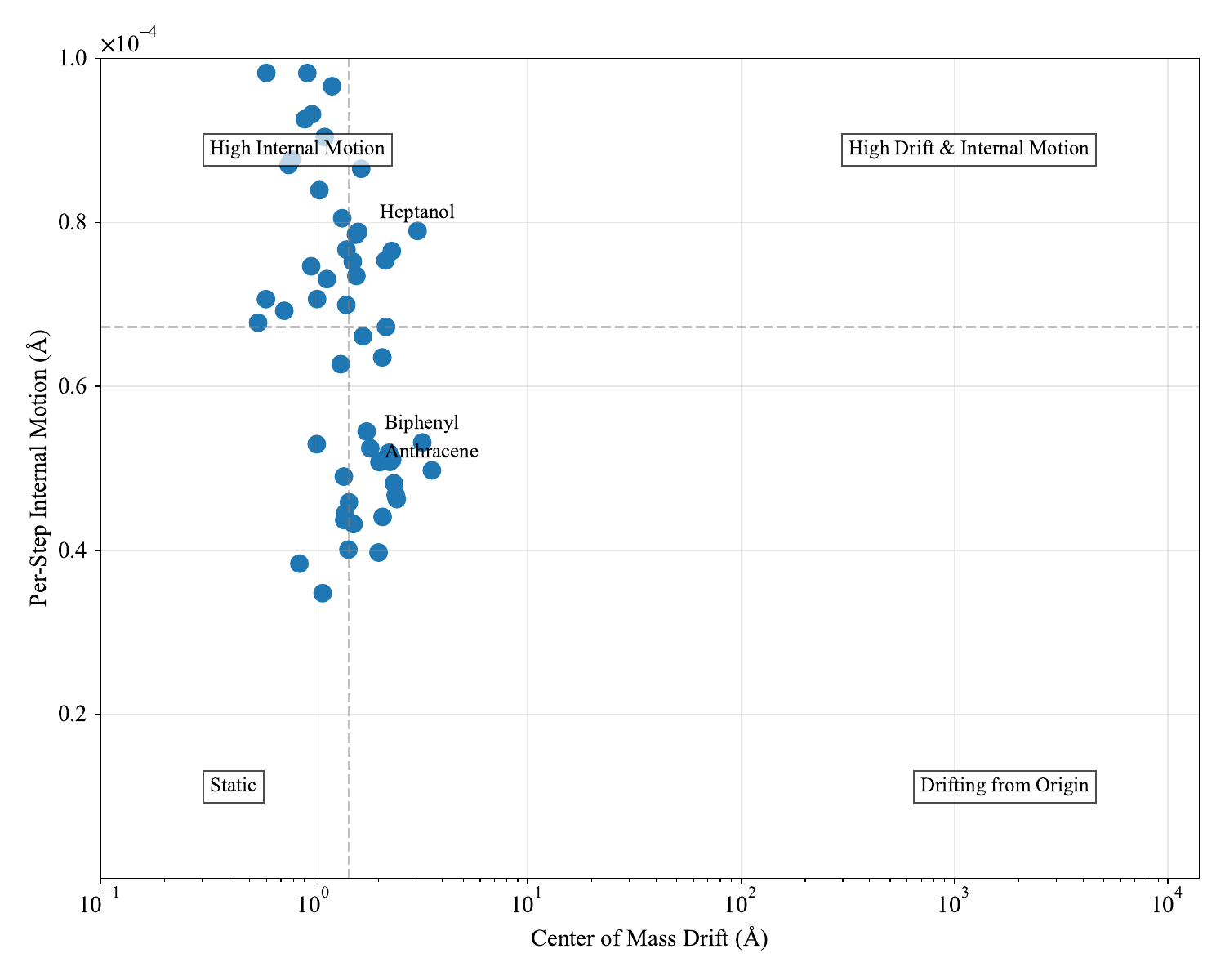}
        \caption{TG80 molecules exhibit the lowest centre-of-mass drift among the evaluated MD datasets.}
        \label{fig:TG80_numerical_quality}
    \end{subfigure}
    \caption{Comparison of numerical stability across MD17, RMD17, and TG80 datasets. Dashed lines denote the mean centre-of-mass drift and per-step motion; datapoints exceeding two standard deviations are annotated.}
    \label{fig:numerical_stability}
\end{figure}

\subsection{TG80 Generation Algorithm}
\label{sec:tg80_algorithm}
{We first recall the definition of Tanimoto \(\operatorname{T}\) similarity between two bit vectors \(X,Y\) as
\[
\operatorname{T}(X,Y) = \frac{|X \cap Y|}{|X \cup Y|},
\]
which is identical to the definition of the Jaccard similarity in this case \cite{rogers_computer_1960}.}

To generate TG80, we randomly shuffled the PubChem dataset, then iterated through all compounds until 40 were found that matched the following criteria:
\begin{wrapfigure}{r}{0.4\linewidth}
  \vspace{1em}
  \centering
  \includegraphics[width=\linewidth]{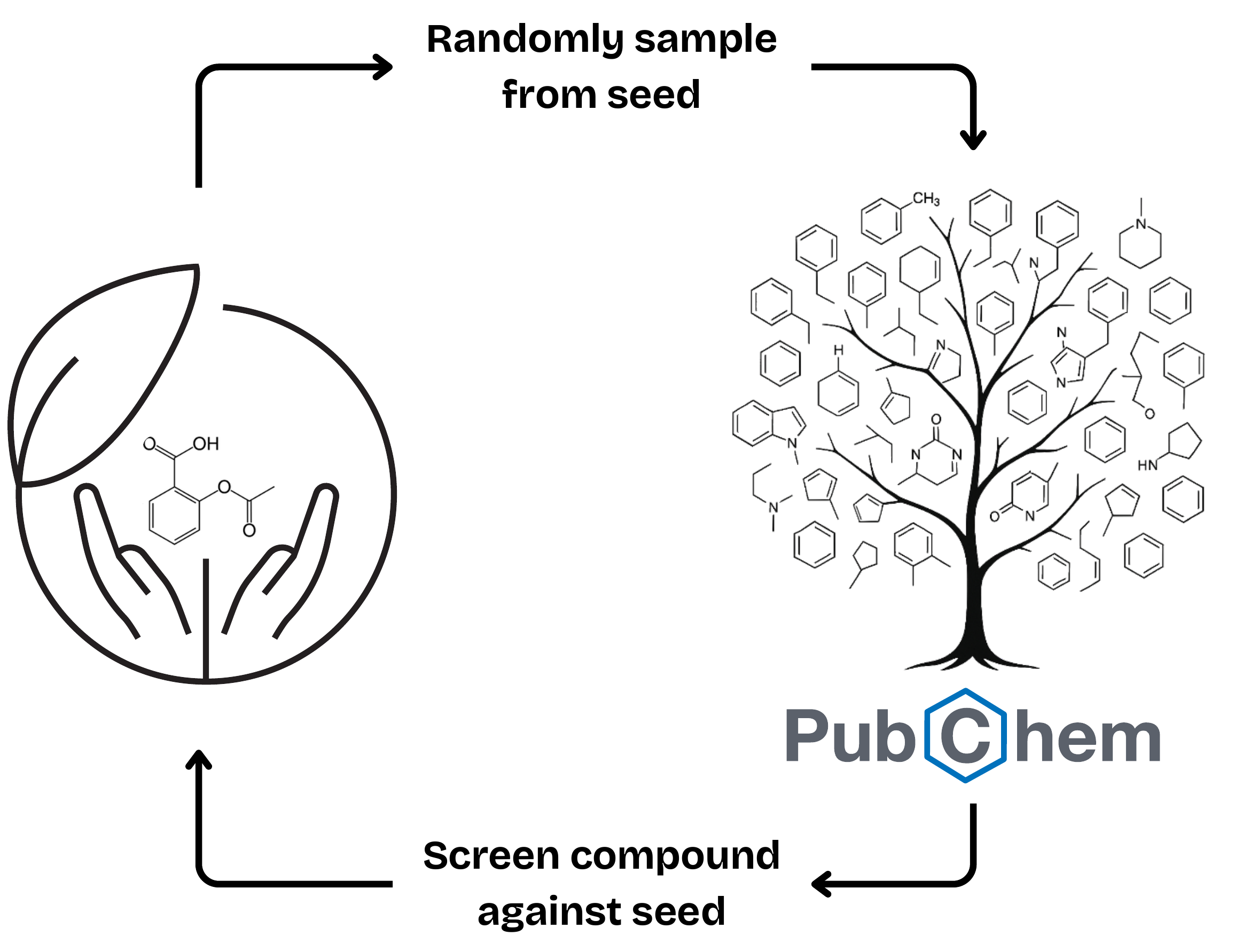}
  \caption{{Construction of TG80 from an initial seed using the PubChem database.}}
  \label{fig:dataset_construction}
  \vspace{-5em}
\end{wrapfigure}
{\begin{enumerate}
    \item \gls{smiles} encode a valid molecular structure
    \item No more heavy atoms than the corresponding seed molecule
    \item Only contain \(\{\text{C, H, O, N}\}\) atoms
    \item No more than five oxygen atoms
    \item No more than three nitrogen atoms
    \item No disconnected molecular fragments (e.g., salts)
    \item Tanimoto similarity to at least one seed molecule greater than 0.875, less than 0.925
    \item Tanimoto similarity to a previously selected molecule is no more than 0.2
\end{enumerate}}
This controlled selection procedure generates structurally analogous subsets around each seed molecule whilst preventing convergence to highly similar molecules across different seed groups.

{Only 2,488 of the 173 million in the PubChem library satisfied the filtration criteria above. This low yield largely reflects the cumulative effect of criterion 8: as more molecules are added, it becomes harder to find candidates sufficiently dissimilar to all prior selections. Given that the average Tanimoto similarity to our seed set was just 0.1492, the 0.875 threshold was highly selective. Dataset generation code is available at ANONYMIZED.}

\subsection{Molecular Dynamics Simulations}
We present a complete overview of the \gls{dft} parameters used to generate MD17 \cite{chmiela_machine_2017}, RMD17 \cite{christensen_role_2020}, MD22 \cite{chmiela_accurate_2023}, and TG80.
\begin{table}[H]
    \centering
    \caption{An overview of the methodologies used to generate the \gls{md} datasets featured.}
    \label{tab:dft_techniques}
    \begin{tabular}{lccccc}
        \toprule
        & \gls{dft} Functional & Dispersion Corrections & Basis set & Timestep & Temperature \\
        \midrule
        MD17 & PBE & TS & NAO & 0.5 fs & 500K  \\
        RMD17 & PBE & None & def2-SVP & 0.5 fs & 500K  \\
        MD22 & PBE & MBD & NAO & 1.0 fs & 500K \\
        TG80 & PBE & \(\Delta 4\) & def2-SVP & 1.0 fs & 300K \\
        \bottomrule
    \end{tabular}
\end{table}

\newpage
\section{Architectural Details}
In this section, we discuss the finer architectural details of \gls{atoms} and the architectural adjustments made for our ablations.
\subsection{Fully Equivariant ATOM}
\label{sec:fully_equivariant}
To achieve the full equivariance discussed in \Cref{fig:MD17_ablation}, we employ a canonicalization network approach, which removes Euclidean gauge before learning and then reinstates it afterwards \cite{kaba_equivariance_2023}. This preserves equivariance of the whole network, even with the use of non-equivariant architectures in the trunk. This provides a controlled comparison between equivariant and non-equivariant \gls{atoms}: it enforces equivariance by construction while otherwise keeping the core architecture unchanged.

We first make data translation equivalent by centering
\begin{equation}
    \mu = \frac{1}{N}\sum_{i=1}^N x_i,
    \qquad
    \bar x_i = x_i - \mu.
\end{equation}
We then remove rotations by aligning to the second moment
\begin{equation}
    S = \tfrac{1}{N}\sum_{i=1}^N \bar x_i \bar x_i^\top
    = \sum_{k=1}^3 \lambda_k e_k e_k^\top
    \quad (\lambda_1 \ge \lambda_2 \ge \lambda_3),
\end{equation}
and choose \(e_1\) as the principal axis and orthonormalise
\begin{equation}
e_2 \leftarrow \frac{e_2 - (e_2^\top e_1)e_1}{\|e_2 - (e_2^\top e_1)e_1\|}, \qquad e_3  = e_1 \times e_2.
\end{equation}
We can then form \(Q = [e_1, e_2, e_3] \in \mathrm{SO}(3)\) and canonicalise
\begin{equation}
    \tilde x_i = (x_i-\mu)\,Q,\qquad \tilde v_i = v_i\,Q.
\end{equation}
We fix the eigenvector sign ambiguity using the chirality pseudoscalar \(c_0=\sum_{i=1}^N x_i\times v_i\) at the reference time (flip \(e_1\) to satisfy \(e_1^\top c_0\ge 0\), then adjust \(e_2,e_3\) jointly to keep right-handedness).
Let \(F\) be an arbitrary trunk acting in the canonical frame; with per-atom canonical outputs \(\hat y_i=F(\{\tilde x_j,\tilde v_j\}_{j=1}^N)_i\), we decanonicalise by
\begin{equation}
    y_i=\hat y_i\,Q^\top+\mu.
\end{equation}

This results in exact \(\mathrm{SE}(3)\)-equivariance \citep{kaba_equivariance_2023} and permits non-equivariant trunks.

\subsection{Random-Walk Positional Encodings}
\label{sec:positional_encoding}
In the multitask case, we add row-normalized \gls{rwpe} to equip \gls{atoms} and \gls{egno} with multiscale connectivity features, enhancing their ability to distinguish non-isomorphic graphs \cite{dwivedi_graph_2022, ma_graph_2023}. We first form a \(\varepsilon\)-neighborhood graph from our pointclouds as:
\begin{equation}
    G = (V,E), \quad V = \{i\}, \quad E = \{(i,j): \|(x,y,z)_i-(x,y,z)_j\|_2 < \varepsilon\}.
\end{equation}
We set \(\varepsilon = 1.6\), as covalent bonds typically range from 1.14 \AA{} to 2.0 \AA{} in length \cite{lobato_highs_2021} and highlight that this construction does not necessitate prior knowledge of the graph structure.

Let $\mathbf{A} \in \mathbb{R}^{n \times n}$ denote the adjacency matrix of this graph, and let $\mathbf{D} = \text{diag}(\mathbf{A}\mathbf{1})$ represent its degree matrix. We construct the random walk transition matrix as \(\mathbf{M} = \mathbf{D}^{-1} \mathbf{A}\) then compute matrix powers of $\mathbf{M}$ up to a maximum walk-length $K$, defining the self-return probabilities for each node as
\begin{equation}
    p_i^{(k)} = \left(\mathbf{M}^k\right)_{ii}, \quad k = 1, \dots, K.
\end{equation}

These probabilities are collected into vectors $\mathbf{p}_i \in \mathbb{R}^{K}$ and concatenated with the phase space to form \(\hat{\mathbf{z}} = \left(\mathbf{v} \parallel \mathcal{Z} \parallel \mathbf{p}\right) \in W_\text{in}\). Here, the input feature space is redefined as \(W_\text{in} = V_\text{in} \oplus \rho_0^\text{even} \oplus \left(\rho_0^\text{even} \otimes \mathbb{R}^{K}\right)\), and the subsequent equivariant maps are modified in kind.

\subsection{Value-residual Learning}
\label{sec:value_residuals}
We employ value-residual learning wherein each transformer block receives the output of the first block via a residual connection to stabilize training and information flow through the network \cite{zhou_value_2024}. Inspired by \cite{jordan_nanogpt_2024}, we add a learned coefficient to weight this residual. Here, \(v\) denotes the current block's value output, and \(v_1\) represents the initial block's value. A learnable parameter \(\alpha\) is passed through a sigmoid to obtain the weighting coefficient:
\begin{equation}
    \lambda = \sigma(\alpha).
\end{equation}
The combined output is then given by:
\begin{equation}
    v = \lambda\,v + (1 - \lambda)\,v_1.
\end{equation}
In practice, we lock the first block's \(\lambda\) value to \(0.5\).

\subsection{Delta-prediction}
\label{sec:delta_prediction}
When delta-prediction is enabled, as in \Cref{fig:MD17_ablation}, we incorporate the initial positions \(\mathbf{x}\) as a residual term, reformulating the model as an operator that learns a displacement field rather than predicting absolute positions. We express this as:
\begin{equation}
    \mathbf{x}^\dagger = \operatorname{Project}(\mathbf{x}_{\text{out}}) + \mathbf{x}.
\end{equation}
Although this approach is implemented in both \gls{egnn} and \gls{egno}, we found it was disabled by default in the codebase of the latter \cite{garcia_satorras_en_2021, xu_equivariant_2024}. Based on empirical results from our ablations, \Cref{fig:MD17_ablation}, we argue there is sufficient evidence to discourage the use of delta-prediction in neural operator-based molecular dynamics simulations.

\section{Further Experiments}
\label{app:further_experiment}
We conduct further experiments on single-task and multitask learning. We consider performance under tail discretization and report results on the RMD17 dataset. For multitask learning, we report performance under random cluster assignment and S2S metrics for the experiments in \Cref{sec:multitask_learning}.

\subsection{Further Single-task Learning Experiments}

\label{app:tail_discretization}
\textbf{MD17 with tail discretization} We find the performance of both \gls{egno} and \gls{atoms} on MD17 with tail discretization remains similar to the performance under uniform discretization discussed in \cref{tab:md17_results}.

\begin{table}[H]
    \centering
    \caption{EGNO and ATOM with final frame sampling. Upper part: S2S \gls{mse}. Lower part: S2T \gls{mse}.}
    \label{tab:md17_finalframe_results}
    \resizebox{\linewidth}{!}{
    \begin{tabular}{lcccccccc}
        \toprule
         & Aspirin & Benzene & Ethanol & Malonaldehyde & Naphthalene & Salicylic & Toluene & Uracil \\
        \midrule
        \gls{egno} & \(9.66{\scriptstyle \pm0.12}\) & \(39.09{\scriptstyle \pm2.35}\) & \(4.57{\scriptstyle \pm0.01}\) & \(\mathbf{12.92}{\scriptstyle \pm0.00}\) & \(0.39{\scriptstyle \pm0.00}\) & \(0.88{\scriptstyle \pm0.01}\) & \(10.99{\scriptstyle \pm0.00}\) & \(0.60{\scriptstyle \pm0.00}\) \\
        \gls{atoms} & \(\mathbf{6.38}{\scriptstyle \pm0.17}\) & \(\mathbf{39.03}{\scriptstyle \pm3.32}\) & \(\mathbf{3.62}{\scriptstyle \pm0.08}\) & \(15.26{\scriptstyle \pm0.65}\) & \(\mathbf{0.39}{\scriptstyle \pm0.00}\) & \(\mathbf{0.83}{\scriptstyle \pm0.01}\) & \(\mathbf{5.26}{\scriptstyle \pm0.79}\) & \(\mathbf{0.55}{\scriptstyle \pm0.00}\) \\
        \midrule
        \rowcolor{gray!20} Gap & \(+33.97\%\) & \(+0.15\%\) & \(+20.85\%\) & \(-18.06\%\) & \(+1.62\%\) & \(+4.75\%\) & \(+52.13\%\) & \(+9.28\%\) \\
        \midrule
        \gls{egno} & \(9.66{\scriptstyle \pm0.11}\) & \(39.15{\scriptstyle \pm2.28}\) & \(4.57{\scriptstyle \pm0.01}\) & \(\mathbf{12.92}{\scriptstyle \pm0.01}\) & \(0.39{\scriptstyle \pm0.00}\) & \(0.88{\scriptstyle \pm0.01}\) & \(10.99{\scriptstyle \pm0.00}\) & \(0.60{\scriptstyle \pm0.00}\) \\
        \gls{atoms} & \(\mathbf{6.38}{\scriptstyle \pm0.17}\) & \(\mathbf{39.03}{\scriptstyle \pm3.35}\) & \(\mathbf{3.63}{\scriptstyle \pm0.08}\) & \(15.21{\scriptstyle \pm0.60}\) & \(\mathbf{0.38}{\scriptstyle \pm0.00}\) & \(\mathbf{0.83}{\scriptstyle \pm0.01}\) & \(\mathbf{5.27}{\scriptstyle \pm0.79}\) & \(\mathbf{0.55}{\scriptstyle \pm0.00}\) \\
        \midrule
        \rowcolor{gray!20} Gap & \(+33.91\%\) & \(+0.30\%\) & \(+20.66\%\) & \(-17.71\%\) & \(+1.82\%\) & \(+5.02\%\) & \(+52.08\%\) & \(+9.44\%\) \\
        \bottomrule
    \end{tabular}
    }
\end{table}

\textbf{Revised MD17 Dataset}
\label{sect:rmd17}
We reach performance parity with \gls{egno} on RMD17, shown in \Cref{tab:rmd17_results}.
\nopagebreak
\begin{table}[H]
    \centering
    \caption{EGNO and ATOM with final frame sampling. Upper part: S2S \gls{mse}. Lower part: S2T \gls{mse}.}
    \label{tab:rmd17_finalframe_results}
    \resizebox{\linewidth}{!}{
    \begin{tabular}{lcccccccc}
        \toprule
         & Azobenzene & Ethanol & Malonaldehyde & Naphthalene & Paracetamol & Salicylic & Toluene & Uracil \\
        \midrule
        \gls{egno} & \(8.96{\scriptstyle \pm0.03}\) & \(\mathbf{23.26}{\scriptstyle \pm0.01}\) & \(\mathbf{40.11}{\scriptstyle \pm0.05}\) & \(1.42{\scriptstyle \pm0.00}\) & \(\mathbf{28.08}{\scriptstyle \pm0.01}\) & \(1.06{\scriptstyle \pm0.01}\) & \(\mathbf{28.28}{\scriptstyle \pm0.01}\) & \(0.88{\scriptstyle \pm0.00}\) \\
        \gls{atoms} & \(\mathbf{8.88}{\scriptstyle \pm0.05}\) & \(23.49{\scriptstyle \pm0.14}\) & \(40.29{\scriptstyle \pm0.13}\) & \(\mathbf{1.36}{\scriptstyle \pm0.00}\) & \(30.12{\scriptstyle \pm0.87}\) & \(\mathbf{1.03}{\scriptstyle \pm0.00}\) & \(28.56{\scriptstyle \pm0.04}\) & \(\mathbf{0.86}{\scriptstyle \pm0.00}\) \\
        \midrule
        \rowcolor{gray!20} Gap & \(+0.90\%\) & \(-0.99\%\) & \(-0.45\%\) & \(+3.93\%\) & \(-7.26\%\) & \(+3.10\%\) & \(-0.99\%\) & \(+1.90\%\) \\
        \midrule
        \gls{egno} & \(8.51{\scriptstyle \pm0.03}\) & \(\mathbf{23.61}{\scriptstyle \pm0.03}\) & \(\mathbf{40.32}{\scriptstyle \pm0.08}\) & \(1.42{\scriptstyle \pm0.00}\) & \(\mathbf{28.01}{\scriptstyle \pm0.02}\) & \(1.07{\scriptstyle \pm0.01}\) & \(\mathbf{28.23}{\scriptstyle \pm0.00}\) & \(0.87{\scriptstyle \pm0.00}\) \\
        \gls{atoms} & \(\mathbf{8.38}{\scriptstyle \pm0.05}\) & \(23.90{\scriptstyle \pm0.15}\) & \(40.67{\scriptstyle \pm0.17}\) & \(\mathbf{1.36}{\scriptstyle \pm0.00}\) & \(30.03{\scriptstyle \pm0.78}\) & \(\mathbf{1.04}{\scriptstyle \pm0.00}\) & \(28.58{\scriptstyle \pm0.05}\) & \(\mathbf{0.85}{\scriptstyle \pm0.00}\) \\
        \midrule
        \rowcolor{gray!20} Gap & \(+1.47\%\) & \(-1.27\%\) & \(-0.88\%\) & \(+4.39\%\) & \(-7.21\%\) & \(+2.78\%\) & \(-1.23\%\) & \(+2.00\%\) \\
        \bottomrule
    \end{tabular}
    }
\end{table}

\subsection{Further Multitask Learning Experiments}

\label{sect:random_split}
\textbf{Random-split cross-validation on TG80.} For completeness, we report multitask results under compound-level random cross-validation, where compounds are randomly assigned to the train, validation, and test sets. Relative to the more challenging out-of-domain (UMAP-based) split in \Cref{tab:tg80_umap}, \gls{egno} is comparatively stronger; nevertheless, \gls{atoms} maintains a consistent lead across folds, with mean improvements of 24.43\% on S2S and 23.93\% on S2T.

\begin{table}[H]
  \centering
  \caption{S2S \gls{mse} \((\times10^{-2})\) on TG80 across five random cluster assignments.}
  \resizebox{\linewidth}{!}{%
  \begin{tabular}{llccccc}
    \toprule
     &  & \textbf{Cluster 1} & \textbf{Cluster 2} & \textbf{Cluster 3} & \textbf{Cluster 4} & \textbf{Cluster 5} \\
     \midrule
    \multirow{2}{*}{OOD}
       & \gls{egno} & \(71.83{\scriptstyle \pm0.00}\) & \(76.92{\scriptstyle \pm0.00}\) & \(68.99{\scriptstyle \pm0.00}\) & \(101.27{\scriptstyle \pm0.00}\) & \(83.20{\scriptstyle \pm0.00}\) \\
       & \gls{atoms} & \(\mathbf{53.93}{\scriptstyle \pm0.00}\) & \(\mathbf{62.40}{\scriptstyle \pm0.00}\) & \(\mathbf{49.37}{\scriptstyle \pm0.00}\) & \(\mathbf{70.75}{\scriptstyle \pm0.00}\) & \(\mathbf{66.75}{\scriptstyle \pm0.00}\) \\
      \cmidrule(lr){2-7}
      \rowcolor{gray!20} & Gap & \(+24.92\%\) & \(+18.88\%\) & \(+28.45\%\) & \(+30.14\%\) & \(+19.77\%\) \\
    \bottomrule
  \end{tabular}
  }
  \label{tab:tg80_random_s2s}
\end{table}

\begin{table}[H]
  \centering
  \caption{S2T \gls{mse} \((\times10^{-2})\) on TG80 across five random cluster assignments.}
  \resizebox{\linewidth}{!}{%
  \begin{tabular}{llccccc}
    \toprule
     &  & \textbf{Cluster 1} & \textbf{Cluster 2} & \textbf{Cluster 3} & \textbf{Cluster 4} & \textbf{Cluster 5} \\
     \midrule
    \multirow{2}{*}{OOD}
       & \gls{egno} & \(63.23{\scriptstyle \pm0.00}\) & \(64.49{\scriptstyle \pm0.00}\) & \(59.18{\scriptstyle \pm0.00}\) & \(85.87{\scriptstyle \pm0.00}\) & \(69.46{\scriptstyle \pm0.00}\) \\
       & \gls{atoms} & \(\mathbf{46.09}{\scriptstyle \pm0.00}\) & \(\mathbf{54.47}{\scriptstyle \pm0.00}\) & \(\mathbf{42.90}{\scriptstyle \pm0.00}\) & \(\mathbf{55.64}{\scriptstyle \pm0.00}\) & \(\mathbf{59.55}{\scriptstyle \pm0.00}\) \\
      \cmidrule(lr){2-7}
      \rowcolor{gray!20} & Gap & \(+27.10\%\) & \(+15.54\%\) & \(+27.51\%\) & \(+35.21\%\) & \(+14.28\%\) \\
    \bottomrule
  \end{tabular}
  }
  \label{tab:tg80_random_s2t}
\end{table}

\textbf{Multitask S2S results on TG80 under UMAP cluster cross-validation.} The S2S side of the multitask learning results follow closely from their S2T counterparts presented in \Cref{sec:multitask_learning}.
\label{app:s2s_result}

\begin{table}[H]
  \centering
  \caption{S2S \gls{mse} \((\times10^{-2})\) on TG80 across five UMAP cluster assignments.}
  \resizebox{\linewidth}{!}{%
  \begin{tabular}{llccccc}
    \toprule
     &  & \textbf{Cluster 1} & \textbf{Cluster 2} & \textbf{Cluster 3} & \textbf{Cluster 4} & \textbf{Cluster 5} \\
     \midrule
     \multirow{3}{*}{ID}
      & \gls{egno}  & \(51.98{\scriptstyle \pm0.81}\) & \(95.86{\scriptstyle \pm0.53}\) & \(142.51{\scriptstyle \pm0.58}\) & \(155.25{\scriptstyle \pm0.67}\) & \(109.25{\scriptstyle \pm0.24}\) \\
      & \gls{atoms} & \(\mathbf{15.49}{\scriptstyle \pm1.04}\) & \(\mathbf{26.55}{\scriptstyle \pm2.13}\) & \(\mathbf{28.74}{\scriptstyle \pm2.40}\) & \(\mathbf{29.81}{\scriptstyle \pm2.72}\) & \(\mathbf{26.33}{\scriptstyle \pm1.98}\) \\
      \cmidrule(lr){2-7}
      \rowcolor{gray!20} & Gap (\%) & \(70.20\%\) & \(72.30\%\) & \(79.83\%\) & \(80.80\%\) & \(75.90\%\) \\
    \midrule
    \multirow{4}{*}{OOD}
      & \gls{egno}  & \(52.90{\scriptstyle \pm0.72}\) & \(114.14{\scriptstyle \pm13.21}\) & \(149.99{\scriptstyle \pm0.34}\) & \(163.47{\scriptstyle \pm1.00}\) & \(112.36{\scriptstyle \pm1.90}\) \\
      & EGNN\text{-}S & \(52.39{\scriptstyle \pm0.40}\) & \(16512.07{\scriptstyle \pm12314.09}\) & \(149.41{\scriptstyle \pm0.94}\) & \(663.54{\scriptstyle \pm865.23}\) & \(111.08{\scriptstyle \pm0.62}\) \\
      & EGNN\text{-}R & \(52.08{\scriptstyle \pm0.79}\) & \(\mathbf{108.89}{\scriptstyle \pm1.60}\) & \(148.67{\scriptstyle \pm0.73}\) & \(163.27{\scriptstyle \pm0.16}\) & \(109.94{\scriptstyle \pm0.31}\) \\
      & \gls{atoms} & \(\mathbf{41.97}{\scriptstyle \pm1.24}\) & \(127.95{\scriptstyle \pm122.67}\) & \(\mathbf{74.53}{\scriptstyle \pm4.82}\) & \(\mathbf{80.95}{\scriptstyle \pm1.21}\) & \(\mathbf{58.26}{\scriptstyle \pm1.68}\) \\
      \cmidrule(lr){2-7}
      \rowcolor{gray!20} & Gap (\%) & \(19.41\%\) & \(-17.50\%\) & \(49.87\%\) & \(50.42\%\) & \(47.01\%\) \\
    \bottomrule
  \end{tabular}
  }
  \label{tab:tg80_umap_s2s}
\end{table}

\newpage

\begin{wraptable}{r}{0.42\textwidth}
    \centering
    \caption{\gls{mse} \((\times10^{-2})\) on full-dataset TG80 \gls{atoms} and single-task \gls{atoms}. S2S upper, S2T lower.}
    \begin{tabular}{lc}
        \toprule
        & Formic Acid \\
        \midrule
        Single-task \gls{atoms} & 26.40 \\
        All-data \gls{atoms} & \textbf{22.39} \\
        \midrule
        \rowcolor{gray!20} Gap (\%) & \(15.19\%\) \\
        \midrule
        Single-task \gls{atoms} & \textbf{18.10} \\
        All-data \gls{atoms} & 18.72 \\
        \midrule
        \rowcolor{gray!20} Gap (\%) & \(-3.43\%\) \\
        \bottomrule
    \end{tabular}
    \label{tab:placeholder}
    \vspace{-0.5\baselineskip}
\end{wraptable}
\textbf{Multitask S2S Versus Single Task} We evaluate a multitask \gls{atoms} model trained on all available trajectories, pooling both ID and OOD clusters, against a single-task \gls{atoms} trained separately on each compound. The multitasking model achieves losses comparable to or lower than those of the single-task baselines, despite being trained with the same compute resources.



\subsection[Monte Carlo Estimation of Quasi-equivariance]{Monte Carlo Estimation of\\ Quasi-equivariance}
\label{app:quasi_equivariance_mc}
{For both the pretrained \gls{atoms} and the non-equivariant lifting variant, we estimate the quantity in \Cref{def:quasi_equivariance} by Monte Carlo, drawing \(N\) random timesteps \(x_n\) and, for each, \(R\) random rotations \(g_{n,r} \in G\). We approximate
\begin{equation*}
    \varepsilon \approx \frac{1}{N} \sum_{n=1}^{N} 
    \left\|
        \frac{1}{R} \sum_{r=1}^{R}
        \bigl(
            f\bigl(\phi(g_{n,r})(x_n)\bigr)
            - \rho(g_{n,r})\bigl(f(x_n)\bigr)
        \bigr)
    \right\|_2 .
\end{equation*}
We report our estimates of \(\varepsilon\) for various \gls{atoms} models trained on MD17 single task learning in \Cref{tab:equivariance_defect}. In all cases, \gls{atoms} shows a substantially lower equivariance defect, supporting our claim that our quasi-equivariant design achieves a kind of middle-ground in the trade-off between expressiveness and strict equivariance.
}

\begin{table}[H]
\centering
\caption{Estimates of Quasi-equivariance \(\varepsilon\) via Monte Carlo over 20 rotations and 10 timesteps with 2SD intervals.}
\label{tab:equivariance_defect}
\resizebox{\linewidth}{!}{%
    \begin{tabular}{lccccccc}
        \toprule
         & Aspirin & Ethanol & Malonaldehyde & Naphthalene & Salicylic & Toluene & Uracil \\
        \midrule 
        Non-equivariant \gls{atoms} & \(120.23\scriptstyle \pm153.71\) & \(102.50{\scriptstyle \pm 45.14}\) & \(26.94{\scriptstyle \pm16.04}\) & \(37.97{\scriptstyle \pm12.65}\) & \(37.81\scriptstyle \pm20.08\) & \(57822.95\scriptstyle \pm42872.97\) & \(32.29{\scriptstyle \pm35.80}\) \\

        \textbf{\gls{atoms}} & \(\mathbf{28.34}{\scriptstyle \pm25.10}\) & \(\mathbf{11.09}{\scriptstyle \pm9.64}\) & \(\mathbf{9.99}{\scriptstyle \pm9.82}\) & \(\mathbf{17.09}{\scriptstyle \pm13.03}\) & \(\mathbf{25.46}{\scriptstyle \pm21.47}\) & \(\mathbf{2744.07}{\scriptstyle \pm3108.52}\) & \(\mathbf{28.72}{\scriptstyle \pm19.20}\) \\
        \bottomrule
    \end{tabular}
}
\end{table}

\section{Experimental Details}
\label{sect:details}

\subsection{Software and Hardware Details}
All experiments were conducted using Python 3.12, NumPy 2.2.1 \cite{harris_array_2020}, PyTorch 2.5.1 \cite{paszke_pytorch_2019}, e3nn 0.5.6 \cite{geiger_e3nn_2022} and PyTorch Optimizer 3.5.0 \cite{kim_pytorch_optimizer_2021}. We use RDKit 2024.9.6 \cite{greg_landrum_rdkitrdkit_2025} and PubChemPy 1.0.4 to construct TG80. All single-task training was performed on an NVIDIA RTX 5080 (16 GB) with CUDA 12.4, running on Ubuntu 24.04. {We use Ase 3.26.0 \cite{larsen_atomic_2017} and MACE-Torch 0.3.14 \cite{kovacs_mace-off_2025} in the experiments of \Cref{app:inference_times}.}

\subsection{Training Times and Compute Requirements}
\label{sec:model_compute_time}
\textbf{Single-task training time} We roughly wall-clock normalised our \gls{atoms} and \gls{egno} parameter counts, resulting in respective learnable parameter counts of 754\,468 and 335\,770.

\begin{table}[H]
    \centering
    \caption{Compute cost of single-task training on all MD17 molecules over 1000 epochs. Both \gls{atoms} (335\,770 params) and \gls{egno} (754\,468 params) are under \texttt{torch.compile} on a Titan V.}
    \label{tab:md17_compute_cost}
    \resizebox{\linewidth}{!}{%
    \begin{tabular}{llcccccccc}
            \toprule
            Model & Metric & Azobenzene & Ethanol & Malonaldehyde & Naphthalene & Paracetamol & Salicylic & Toluene & Uracil \\
            \midrule
            \multirow{3}{*}{\gls{egno}}
             & Time (mins) & \(4.09\pm\scriptstyle{0.15}\) & \(3.62\pm\scriptstyle{0.42}\) & \(3.65\pm\scriptstyle{0.01}\) & \(5.01\pm\scriptstyle{0.02}\) & \(9.02\pm\scriptstyle{1.22}\) & \(5.16\pm\scriptstyle{0.03}\) & \(3.93\pm\scriptstyle{0.02}\) & \(4.35\pm\scriptstyle{0.04}\) \\
             & Total FLOPS (\(\times 10^{12}\)) & \(3681.24\) & \(3257.70\) & \(3282.09\) & \(4513.37\) & \(8114.44\) & \(4641.50\) & \(3539.92\) & \(3915.98\) \\
             & Epochs/min & \(244.48\) & \(276.27\) & \(274.22\) & \(199.41\) & \(110.91\) & \(193.90\) & \(254.24\) & \(229.83\) \\
            \midrule
            \multirow{3}{*}{\gls{atoms}}
             & Time (mins) & \(5.81\pm\scriptstyle{0.02}\) & \(5.79\pm\scriptstyle{0.06}\) & \(5.79\pm\scriptstyle{0.00}\) & \(5.86\pm\scriptstyle{0.01}\) & \(5.89\pm\scriptstyle{0.02}\) & \(5.85\pm\scriptstyle{0.01}\) & \(5.81\pm\scriptstyle{0.01}\) & \(5.83\pm\scriptstyle{0.02}\) \\
             & Total FLOPS (\(\times 10^{12}\)) & \(5226.49\) & \(5212.53\) & \(5213.50\) & \(5271.19\) & \(5297.84\) & \(5263.76\) & \(5224.91\) & \(5247.33\) \\
             & Epochs/min & \(172.20\) & \(172.66\) & \(172.63\) & \(170.74\) & \(169.88\) & \(170.98\) & \(172.25\) & \(171.52\) \\
            \midrule
            \rowcolor{gray!20}
            \multicolumn{2}{c}{Total FLOPS Reduction (\%)}  & \(-41.98\%\) & \(-60.01\%\) & \(-58.85\%\) & \(-16.79\%\) & \(+34.71\%\) & \(-13.41\%\) & \(-47.60\%\) & \(-34.00\%\) \\
            \bottomrule
        \end{tabular}
    }
\end{table}

\textbf{Multitask training time} In multitask training on TG80, our upsized \gls{atoms} model contained 3\,557\,840 parameters, compared to 335 770 for \gls{egno}. Despite this, \gls{atoms} only trained between 5\% and 30\% slower than \gls{egno}. This is perhaps unsurprising given the much higher FLOPS-utilization of the transformer architecture upon which \gls{atoms} is based.

\begin{table}[H]
    \centering
    \caption{Compute cost of single task training on five TG80 molecules over 1000 epochs. Both \gls{atoms} (335 770 params) and \gls{egno} (754 468 params) are under \texttt{torch.compile} on a Titan V.}
    \label{tab:tg80_compute_cost}
    \resizebox{\linewidth}{!}{%
    \begin{tabular}{llccccc}
            \toprule
            Model &  & Fold1 & Fold2 & Fold3 & Fold4 & Fold5 \\
            \midrule
            \multirow{3}{*}{\gls{egno}} & Time (mins) & \(9.61{\scriptstyle \pm1.21}\) & \(8.56{\scriptstyle \pm0.06}\) & \(9.04{\scriptstyle \pm0.11}\) & \(9.31{\scriptstyle \pm0.20}\) & \(8.98{\scriptstyle \pm0.01}\) \\
             & Total FLOPS ($\times10^{12}$) & 8645.66 & 7703.33 & 8136.98 & 8378.06 & 8084.98 \\
             & Epochs/min & \(104.10\) & \(116.83\) & \(110.61\) & \(107.42\) & \(111.32\) \\
            \midrule
            \multirow{3}{*}{\gls{atoms}} & Time (mins) & \(10.16{\scriptstyle \pm0.49}\) & \(10.55{\scriptstyle \pm0.02}\) & \(11.62{\scriptstyle \pm0.41}\) & \(12.38{\scriptstyle \pm0.12}\) & \(10.39{\scriptstyle \pm0.41}\) \\
             & Total FLOPS ($\times10^{12}$) & 9140.57 & 9497.79 & 10455.34 & 11141.73 & 9355.37 \\
             & Epochs/min & \(98.46\) & \(94.76\) & \(86.08\) & \(80.78\) & \(96.20\) \\
            \midrule
            \rowcolor{gray!20}
            \multicolumn{2}{c}{Total FLOPS Reduction (\%)}  & \(-5.72\%\) & \(-23.29\%\) & \(-28.49\%\) & \(-32.99\%\) & \(-15.71\%\) \\
            \bottomrule
        \end{tabular}
    }
\end{table}
\subsection{Inference Times}
\label{app:inference_times}
{We compare inference times on MD17 and MD22 across \gls{atoms}, the pretrained machine learning interaction potential MACE-OFF24 (Medium) \cite{kovacs_mace-off_2025}, and the classical Lennard-Jones potential \cite{larsen_atomic_2017, schwerdtfeger_100_2024} and the molecular forcefield. We report inference times in seconds with 2SD intervals. We exclude AMBER results on Uracil as we were unable to run simulations for this molecule \cite{case_ambertools_2023}.}

\begin{table}[H]
    \centering
    \caption{Seconds to produce timestep \(\Delta T = 3000\) of each MD17 trajectory at \texttt{float32} precision.}
    \label{tab:equivariance_defect}
    \resizebox{\linewidth}{!}{%
        \begin{tabular}{llrrrrrrr}
            \toprule
             &  & Aspirin & Ethanol & Malonaldehyde & Naphthalene & Salicylic & Toluene & Uracil \\
            \midrule 
            \multirow{4}{*}{\(\Delta T = 3000\)} 
              &  MACE-OFF
              & \(42.605\pm\scriptstyle{3.945}\) 
              & \(39.656\pm\scriptstyle{0.568}\) 
              & \(39.782\pm\scriptstyle{0.141}\) 
              & \(39.916\pm\scriptstyle{0.155}\) 
              & \(39.620\pm\scriptstyle{0.030}\) 
              & \(40.271\pm\scriptstyle{2.851}\) 
              & \(40.166\pm\scriptstyle{0.902}\) \\
              &  AMBER
              & \(37.746\pm\scriptstyle{0.536}\) 
              & \(38.512\pm\scriptstyle{1.243}\) 
              & \(36.806\pm\scriptstyle{0.194}\) 
              & \(38.512\pm\scriptstyle{1.243}\) 
              & \(37.621\pm\scriptstyle{1.412}\) 
              & \(38.061\pm\scriptstyle{0.848}\) 
              & \(-\) \\
            
              & Lennard-Jones
              & \(2.499\pm\scriptstyle{0.266}\) 
              & \(1.604\pm\scriptstyle{0.208}\) 
              & \(1.529\pm\scriptstyle{0.039}\) 
              & \(2.174\pm\scriptstyle{0.094}\) 
              & \(1.981\pm\scriptstyle{0.011}\) 
              & \(1.906\pm\scriptstyle{0.006}\) 
              & \(1.804\pm\scriptstyle{0.196}\) \\
    
              & \gls{atoms}
              & \(0.849\pm\scriptstyle{0.926}\) 
              & \(0.259\pm\scriptstyle{0.097}\) 
              & \(0.714\pm\scriptstyle{0.241}\) 
              & \(0.467\pm\scriptstyle{0.149}\) 
              & \(0.450\pm\scriptstyle{0.090}\) 
              & \(0.341\pm\scriptstyle{0.075}\) 
              & \(0.373\pm\scriptstyle{0.076}\) \\
              \midrule
            \multirow{3}{*}{\(\Delta T = 10\ 000\)} 
              & MACE-OFF
              & \(143.413\pm\scriptstyle{3.313}\) 
              & \(142.569\pm\scriptstyle{5.584}\) 
              & \(136.380\pm\scriptstyle{3.964}\) 
              & \(140.288\pm\scriptstyle{9.136}\) 
              & \(140.103\pm\scriptstyle{3.398}\) 
              & \(140.372\pm\scriptstyle{4.754}\) 
              & \(141.598\pm\scriptstyle{1.020}\) \\
              
              &   AMBER
              & \(133.900\pm\scriptstyle{3.815}\) 
              & \(128.385\pm\scriptstyle{3.812}\) 
              & \(121.095\pm\scriptstyle{0.968}\) 
              & \(121.607\pm\scriptstyle{1.195}\) 
              & \(120.319\pm\scriptstyle{1.591}\) 
              & \(120.800\pm\scriptstyle{0.398}\) 
              & \(-\) \\
            
              & Lennard-Jones
              & \(7.797\pm\scriptstyle{0.369}\) 
              & \(5.533\pm\scriptstyle{0.425}\) 
              & \(5.167\pm\scriptstyle{0.221}\) 
              & \(7.453\pm\scriptstyle{0.376}\) 
              & \(6.721\pm\scriptstyle{0.418}\) 
              & \(6.477\pm\scriptstyle{0.204}\) 
              & \(5.921\pm\scriptstyle{0.522}\) \\
    
              & \gls{atoms}
              & \(2.554\pm\scriptstyle{0.945}\) 
              & \(1.002\pm\scriptstyle{0.184}\) 
              & \(1.634\pm\scriptstyle{0.200}\) 
              & \(1.098\pm\scriptstyle{0.209}\) 
              & \(1.116\pm\scriptstyle{0.083}\) 
              & \(0.964\pm\scriptstyle{0.185}\) 
              & \(0.990\pm\scriptstyle{0.171}\) \\
            \bottomrule
        \end{tabular}
    }
\end{table}

\begin{table}[H]
    \centering
    \caption{Seconds to produce timestep \(\Delta T = 3000\) of each MD22 trajectory at \texttt{float32} precision.}
    \label{tab:equivariance_defect}
        \begin{tabular}{lrrr}
            \toprule
             & Ac-Ala3-NHME & DHA & Stachyose \\
            \midrule 
            MACE-OFF & \(44.737\pm\scriptstyle{4.615}\) & \(42.042\pm\scriptstyle{0.720}\) & \(47.913\pm\scriptstyle{0.546}\) \\
            
            Lennard-Jones & \(4.700\pm\scriptstyle{0.212}\) & \(3.790\pm\scriptstyle{0.246}\) & \(6.613\pm\scriptstyle{0.147}\) \\
    
            \gls{atoms} & \(1.302\pm\scriptstyle{0.766}\) & \(0.914\pm\scriptstyle{0.055}\) & \(3.075\pm\scriptstyle{0.080}\) \\
            \bottomrule
        \end{tabular}
\end{table}

\newpage
\subsection{ATOM Hyperparameters}
\label{appx:my_hyperparams}
We employ the same dataset splitting and discretization parameters reported in \textcite{xu_equivariant_2024} for the MD17. We set the batch size to \(192\), use the AdamW-AMSGrad optimizer \cite{loshchilov_decoupled_2017} with an \(\epsilon\) of \(1\times10^{-10}\) to avoid instability associated with the small gradients produced by zero-initialised weight matrices in early training \cite{jordan_nanogpt_2025}. During multitask training, we reduce the number of epochs to 250 and employ the Muon optimizer \cite{jordan_muon_2024, kim_pytorch_optimizer_2021}. We present a complete overview of our hyperparameters in \Cref{tab:atom_hparams}.

\begin{table}[h]
    \centering
    \caption{Hyperparameters for \gls{atoms}. MD17 hyperparameters are shared across all molecules unless otherwise noted.}
    \label{tab:atom_hparams}
    \resizebox{\linewidth}{!}{
    \begin{tabular}{cccc}
        Module & & MD17, RMD17, TG80 & TG80 Multitask \\
        \midrule
        Training & & \\
        & Batch size & \( 192 \) & \( 192 \) \\
        & Epochs & \( 1000 \) & \(250\) \\
        & Max grad norm & \( 1.0 \) & \(1.0\) \\
        & Label noise \(\sigma\) & \( 0.1 \) & \(0.1\) \\
        & \(\Delta t\) & \( 3000 \) & \( 10\ 000 \) \\
        & Timesteps \(P\) & \( 8 \) & \( 8 \) \\
        & Train/Val/Test & \( (500,3\,000,3\,000) \) & \( (6\,500,13\,000,13\,000) \) \\
        & \gls{rwpe} length & 8 & 8 \\
        \midrule
        Optimiser & & \\
        & Optimiser type & AdamW-AMSGrad & Muon  \\
        & Learning rate & \( 1\times10^{-3} \) & \( 1\times10^{-3} \) \\
        & \(\beta_1, \beta_2\) & \( 0.9, 0.999 \) & \( (0.9, 0.999) \) \\
        & Weight decay & \( 1\times10^{-5} \) & \( 1\times10^{-5} \) \\
        & \(\epsilon\) & \( 1\times10^{-10} \) & \( 1\times10^{-5} \) \\
        \midrule
        Model & & \\
        & Embedding dim & \( 128 \) & \( 256 \) \\
        & No. layers & \( 5 \) & \( 6 \) \\
        & No. attention heads & \( 8 \) & \( 8 \) \\
        & No. output heads & \(1\) & \(8\) \\
        & Attention dropout & \( 0.2 \) & \( 0.2 \) \\
        & RoPE frequency & \( 1000 \) & \( 1000 \)\\
        & MLP layers & \( 2 \) & \( 2 \) \\
        & MLP activation & SwiGLU & SwiGLU \\
        & MLP dropout & \( 0.0 \) & \( 0.0 \) \\
        & Norm type & RMS norm & RMS norm \\
        & Learnable value residuals & True & True \\
        \bottomrule
    \end{tabular}}
\end{table}
\newpage
\subsection{EGNO Hyperparameters and Experimental Details}
\label{appx:egno_hyperparams}
We generated the \gls{egno} results reported in \Cref{tab:md17_results} with the same discretization parameters and hyperparameters as used in their experiments. We reduce the number of epochs from \(10\ 000\) to \(2\ 500\), use a batch size of 192 with the AdamW-AMSGrad optimizer \cite{loshchilov_decoupled_2017}, and select the best validation loss epoch for testing. In the multitask case, we further reduce the number of epochs to 250 and employ the Muon optimizer \cite{jordan_muon_2024, kim_pytorch_optimizer_2021}. Complete hyperparameters are displayed in \Cref{tab:egno_hparams}.

\begin{table}[H]
    \centering
    \caption{Hyperparameter values for \gls{egno} across each benchmark dataset.}
    \label{tab:egno_hparams}
    \resizebox{\linewidth}{!}{
    \begin{tabular}{cccc}
        Module & & MD17, RMD17, TG80 & TG80 Multitask \\
        \midrule
        Training & & \\
        & Batch size & \( 192 \) & \( 192 \) \\
        & Epochs & \( 2500 \) & \( 250 \) \\
        & Max grad norm & Uncapped & Uncapped  \\
        & Label noise \(\sigma\) & \( 0.1 \) & \(0.1\) \\
        & \(\Delta t\) & \( 3000 \) & \(10\ 000\) \\
        & Timesteps \(P\) & \( 8 \) & \(8\) \\
        & Train/Val/Test & \( (500,3\,000,3\,000) \) & \( (6\,500,13\,000,13\,000) \) \\
        & \gls{rwpe} length & 8 & 8 \\
        \midrule
        Optimiser & & \\
        & Optimiser type & AdamW-AMSGrad & Muon \\
        & Learning rate & \( 1\times10^{-3} \) & \( 1\times10^{-3} \) \\
        & \(\beta_1, \beta_2\) & \( (0.9, 0.999) \) & \( (0.9, 0.999) \) \\
        & Weight decay & \( 1\times10^{-5} \) & \( 1\times10^{-5} \) \\
        & \(\epsilon\) & \( 1\times10^{-10} \) & \( 1\times10^{-5} \) \\
        \midrule
        Scheduler & & \\
        & Scheduler type & StepLR & StepLR \\
        & Step size & 2500 & 2500 \\
        & \(\gamma\) & 0.5 & 0.5 \\
        \midrule
        Model & & \\
        & Embedding dim & \( 64 \) & \( 64 \) \\
        & No. \gls{egno} layers & \( 5 \) & \( 5 \) \\
        & Temporal convolution activation & LeakyRELU & LeakyRELU \\
        & MLP layers & \( 2 \) & \( 2 \) \\
        & MLP activation & SiLU & SiLU \\
        & MLP dropout & \( 0 \) & \( 0 \) \\
        & Time embedding dim & 32 & 32 \\
        & Fourier modes & 2 & 2 \\
        \bottomrule
    \end{tabular}}
\end{table}

\newpage
\section{Propositions and Proofs}

\subsection{Kernel Integral form of Cross-attention}
\label{app:kernel}

\begin{proposition}
\label{prop:kernel}
The cross-attention is equivalent to a kernel integral operator, i.e., $\operatorname{softmax}(\operatorname{T-RoPE}(\mathbf Q) \operatorname{T-RoPE}(\mathbf K_i)^\top/\sqrt{d_h}) \mathbf V_i = \int \kappa_i(\mathbf{z}, \mathbf{x}) v_i(x) d \mu_N(\mathbf{x})$, where $\kappa_i$ denotes the kernel induced by softmax function, $v_i(\mathbf x)$ denotes the values as a function of $\mathbf x$, and $\mu_N$ denotes the empirical measure supported on $\{ \mathbf{x}_j\}_{j=1}^N$. 
\end{proposition}


\label{appx:kernel_integral_cross_att}
\begin{proof}[Proof of Proposition \ref{prop:kernel}]
Following \cite{gao_learning_2024} we may view our attention as a kernel integral transform by considering \(\mathbf{x}_i\) as being sampled from the continuum domain \(\Omega \subset \mathbb{R}^3\) for which we define the empirical measure with support on \(\{\mathbf{x}_j\}_{j=1}^N \subset \Omega\):
\begin{equation}
    {\mu_N} (\mathbf{x}) = \frac{1}{N} \sum_{i=1}^N \delta_{\mathbf{x}_j}, \qquad \int_\Omega g(\mathbf{x})d {\mu_N } (\mathbf x) =\frac{1}{N} \sum_{j=1}^N g(\mathbf{x}_j)
    \label{eq:measure}
\end{equation}
where \(\delta\) is the Dirac delta function \say{selecting} the values at \(\mathbf{x}_j\).
Given \gls{trope}-rotated query and key maps \(\tilde{q}_\theta(\mathbf{z})=R_{p(\mathbf{z})}q_\theta(\mathbf{z}_j)\), \(\tilde{k}_i(\mathbf{x}_j)=R_{p(\mathbf{x}_j)}k_{\theta,i}(\mathbf{x}_j)\) we form the data-dependent kernel for feature \(F\):
\begin{equation}
    \kappa_{\theta,i}(\mathbf{z}, \mathbf{x}_j)
    = 
    \frac{
      \exp\bigl(\langle \tilde q(\mathbf{z}),\,\tilde k_i(\mathbf{x}_j)\rangle \;/\;\sqrt{d_{h}}\bigr)
    }{
      \displaystyle
      \int_{\Omega}
        \exp\bigl(\langle \tilde q(\mathbf{z}_{j}),\,\tilde k_i(\mathbf{x}')\rangle \;/\;\sqrt{d_{h}}\bigr)
        \,d\mu_{N}(\mathbf{x}')
    }\,.
\end{equation}
Thus, for any \(F \in \mathcal{F}\) we may represent our cross-attention as the kernel integral operator:
\begin{equation}
    \bigl(\mathcal{K}_{\theta,i} \mathbf{v}_j\bigr)(\mathbf{z}) = \int_{\Omega} \kappa_{\theta,i}\bigl(\mathbf{z}, \mathbf{x}\bigr)\,\mathbf{v}_i(\mathbf{x})\,\mathrm{d}\mu_N(\mathbf{x}),
    \qquad
    \int_\Omega\kappa_{\theta,i}(\mathbf{z},\mathbf{x})\, d \mu_N(\mathbf{x})=1,
\end{equation}
which is row-stochastic under the measure in \Cref{eq:measure}.     
\end{proof}

We remark that the kernel fails to satisfy global Lipschitz continuity \cite{delattre_efficient_2023}, unlike \gls{fno} \cite{li_fourier_2021}, and certain generalization theorems fail as a result \cite{le_mathematical_2024}.

\newpage
\section{Trajectory Samples}
\FloatBarrier
\begin{figure}[!ht]
    \centering
    \begin{subfigure}{0.48\linewidth}
        \centering
        \includegraphics[width=\linewidth]{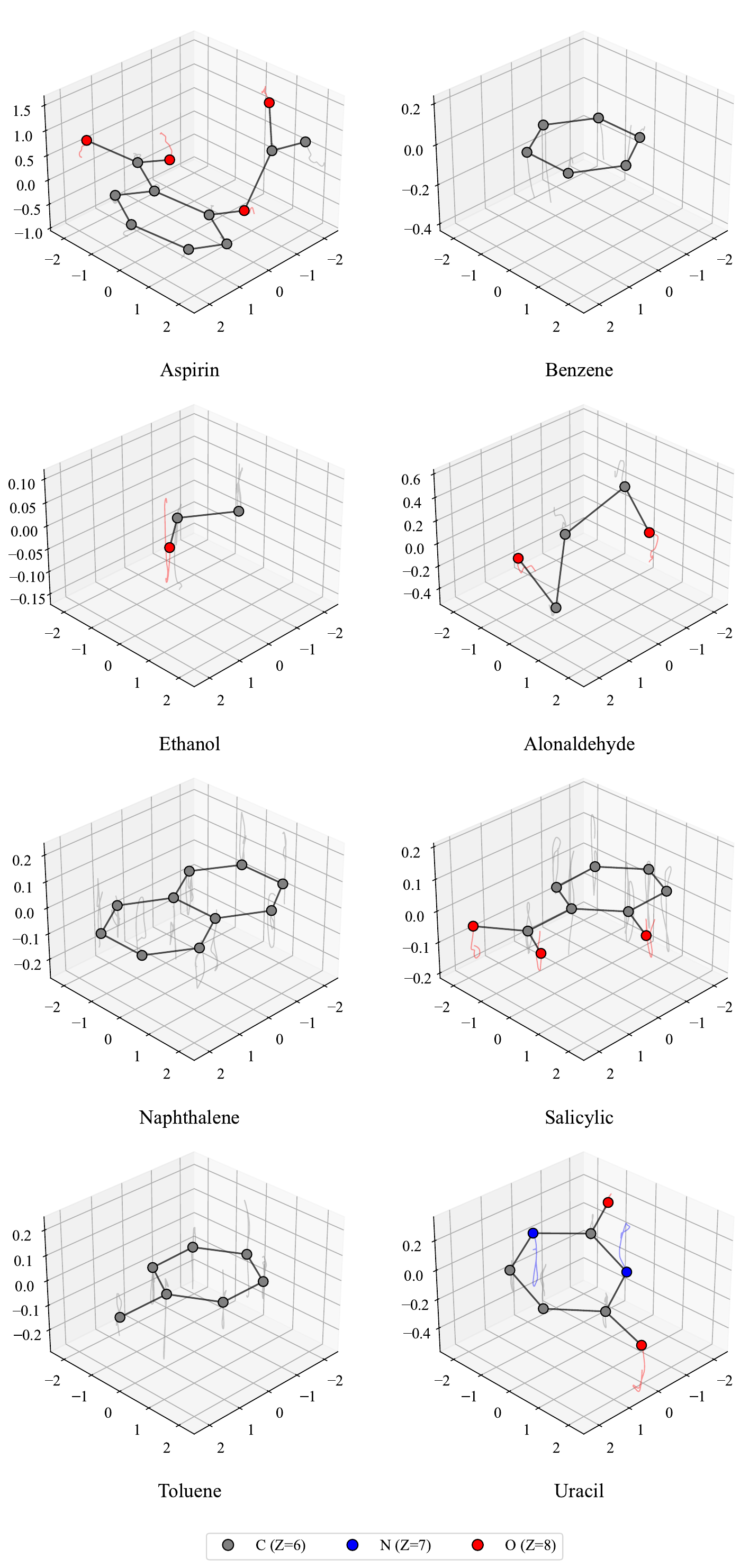}
        \caption{MD17 trajectories}
        \label{fig:MD17_trajectories}
    \end{subfigure}
    \hfill
    \begin{subfigure}{0.48\linewidth}
        \centering
        \includegraphics[width=\linewidth]{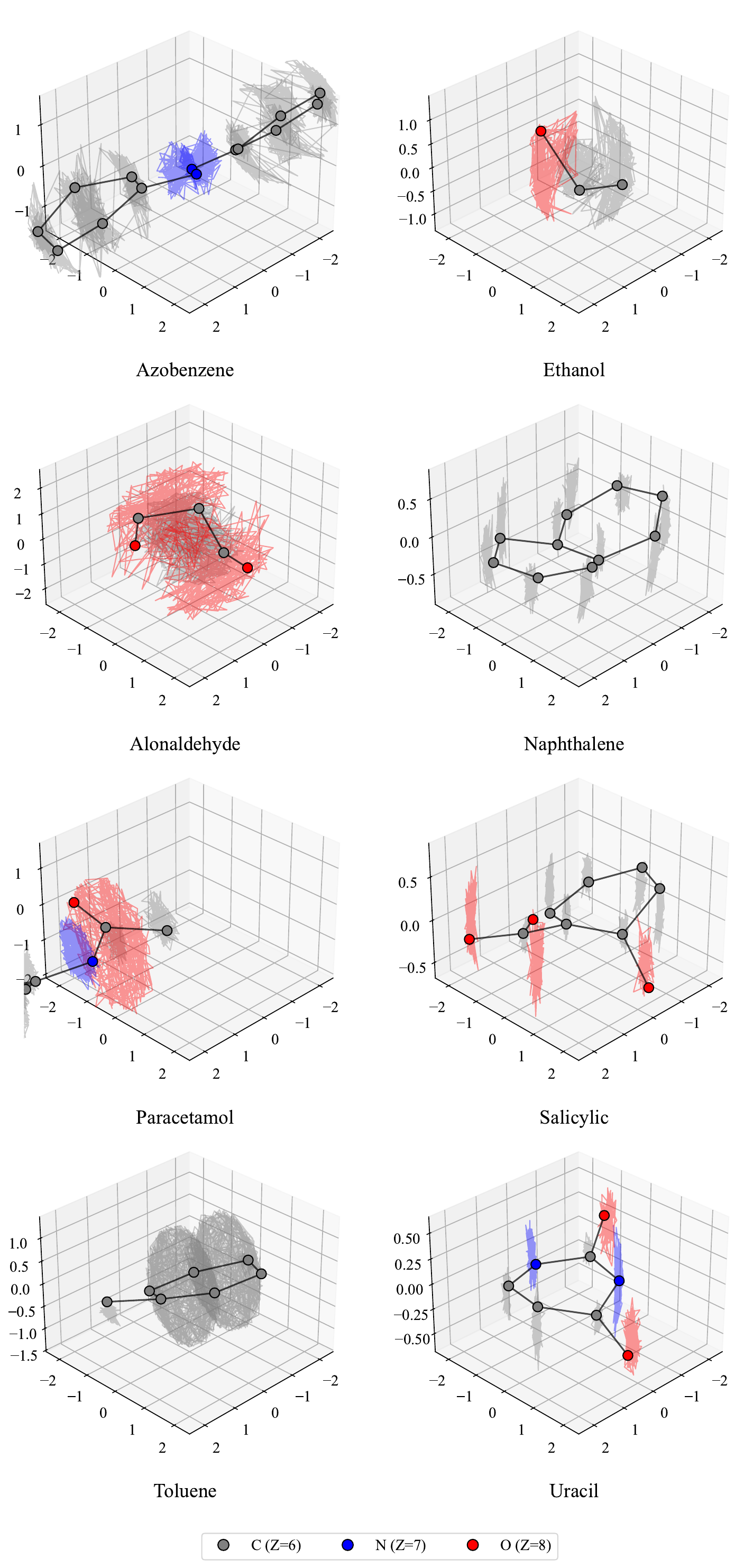}
        \caption{RMD17 trajectories}
        \label{fig:RMD17_trajectories}
    \end{subfigure}
    \caption{3000 steps \gls{md} trajectories from the MD17 and RMD17 datasets.}
    \label{fig:trajectories}
\end{figure}
\newpage
\begin{figure}[H]
    \centering
    \includegraphics[width=\linewidth]{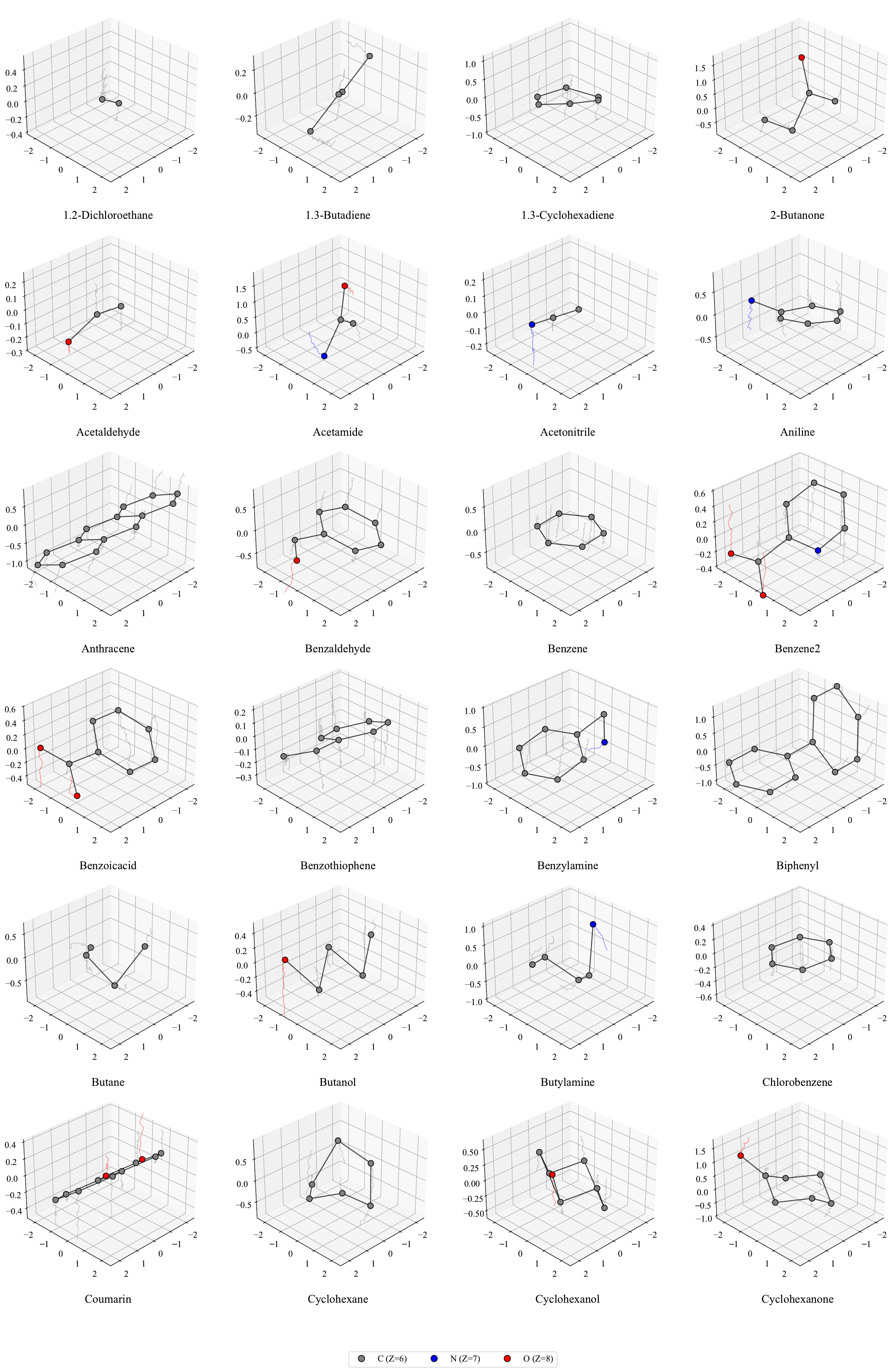}
    \label{fig:TG80_trajectories_1}
\end{figure}
\begin{figure}
    \centering
    \includegraphics[width=\linewidth]{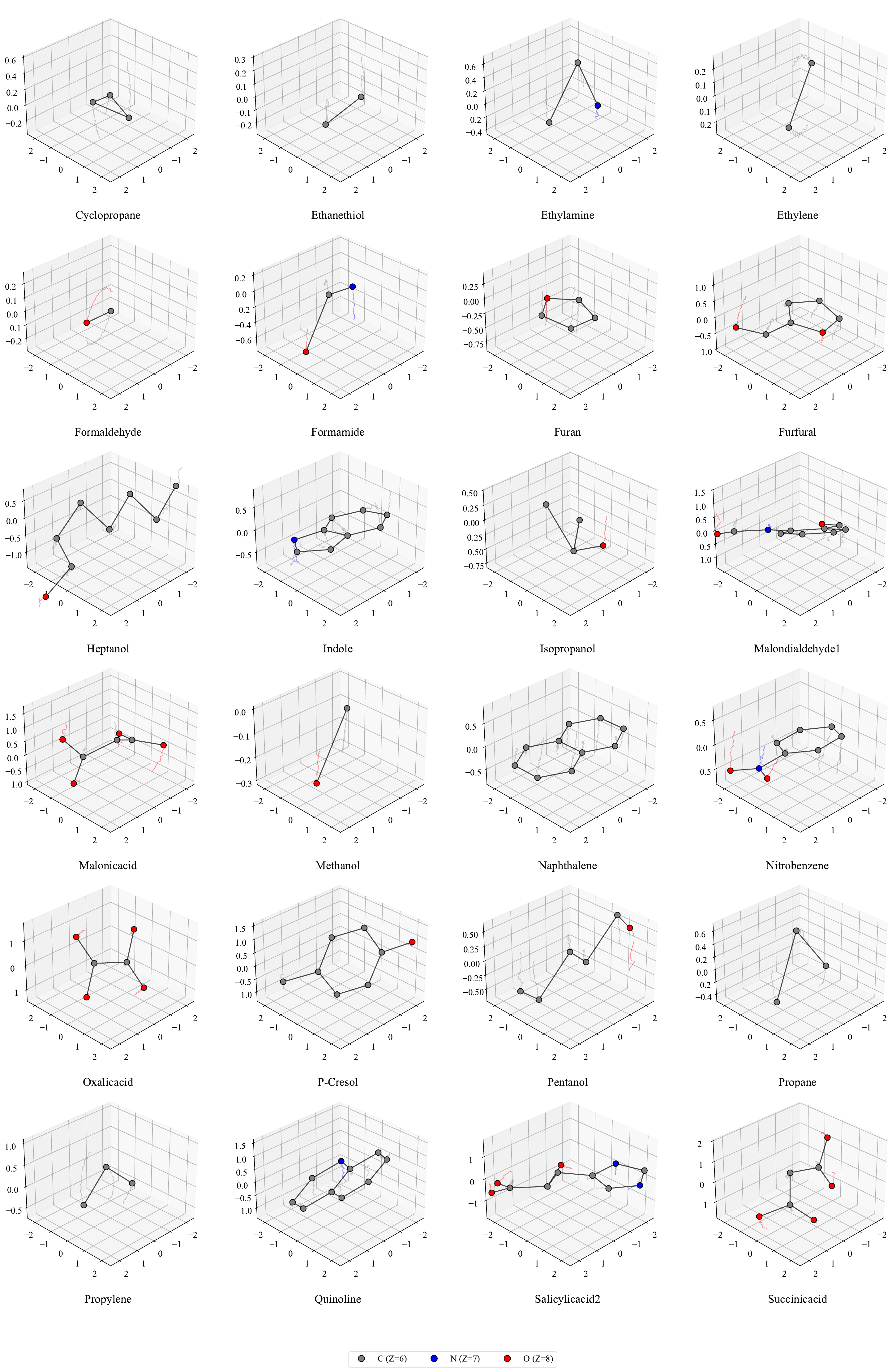}
    \label{fig:TG80_trajectories_2}
\end{figure}
\begin{figure}
    \centering
    \includegraphics[width=\linewidth]{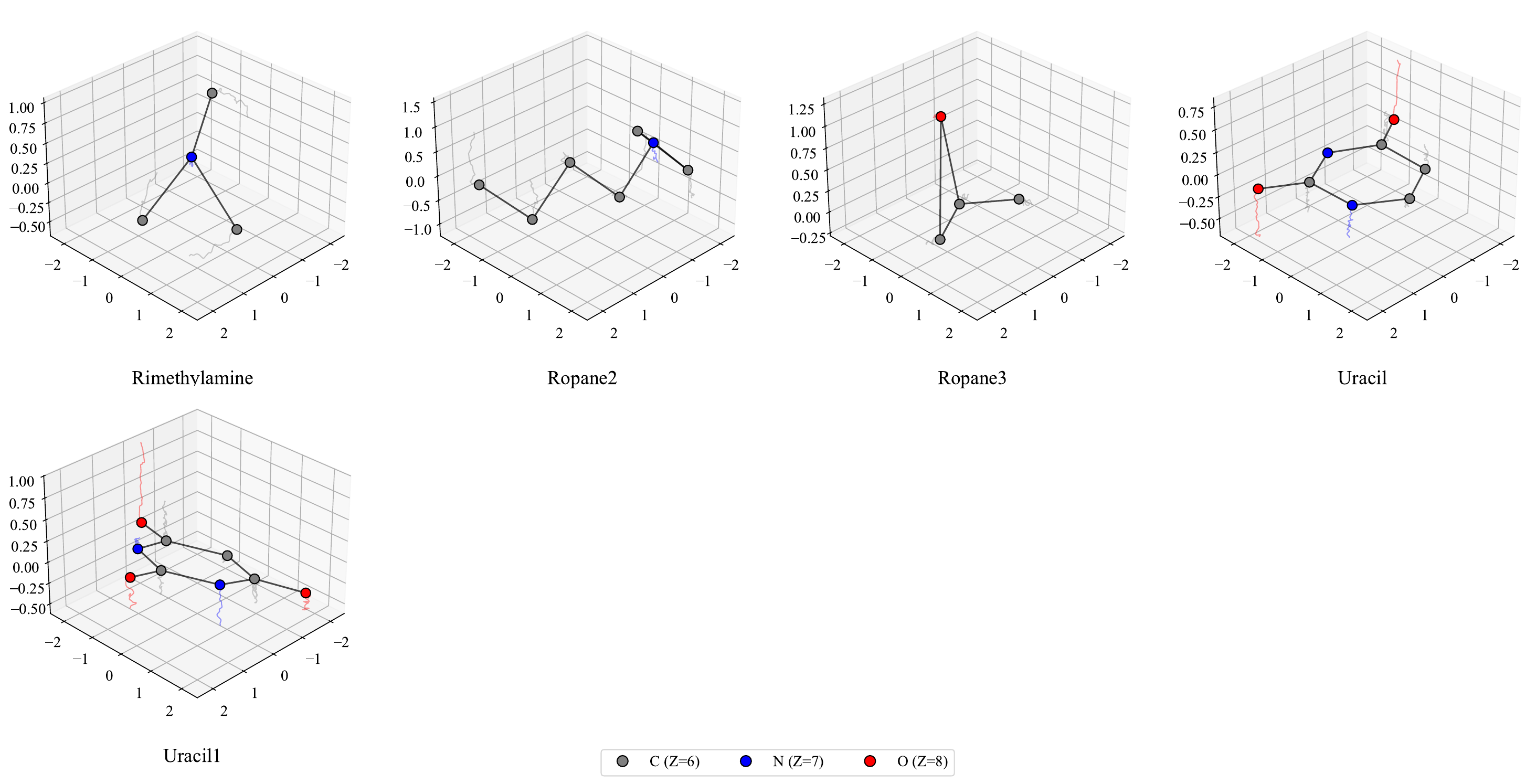}
    \label{fig:TG80_trajectories_3}
    \caption{3000-step \gls{md} trajectories from TG80. Molecules generated by our dataset expansion algorithm are named according to their seed molecule and the order of their selection.}
\end{figure}

\end{document}